\documentclass{article}

\usepackage{microtype}
\usepackage{graphicx}
\usepackage{booktabs} % for professional tables
\usepackage{diagbox}

\usepackage{hyperref}

\usepackage[ruled,vlined]{algorithm2e}

\usepackage[accepted]{icml2024}

\usepackage{amsmath}
\usepackage{amssymb}
\usepackage{mathtools}
\usepackage{amsthm}
\usepackage{wasysym}
\usepackage{subcaption}
\usepackage{nicefrac}\usepackage{xfrac}
\usepackage{siunitx}

\usepackage{caption}
\usepackage{multirow}

\usepackage[capitalize,noabbrev]{cleveref}
\allowdisplaybreaks

\everypar{\looseness=-1}
\theoremstyle{plain}
\newtheorem{theorem}{Theorem}[section]

\theoremstyle{definition}

\theoremstyle{remark}

\newcommand*{\defeq}{\stackrel{\text{def}}{=}}
\newcommand{\KL}[2]{\text{KL}\left(#1\Vert #2\right)}

\usepackage[textsize=tiny]{todonotes}

\DeclareMathOperator*{\argmin}{arg\,min}

\icmltitlerunning{Light and Optimal Schrödinger Bridge Matching}

\begin{document}

\twocolumn[
\icmltitle{Light and Optimal Schrödinger Bridge Matching}



\icmlsetsymbol{equal}{*}

\begin{icmlauthorlist}
\icmlauthor{Nikita Gushchin}{equal,skoltech}
\icmlauthor{Sergei Kholkin}{equal,skoltech}
\icmlauthor{Evgeny Burnaev}{skoltech,airi}
\icmlauthor{Alexander Korotin}{skoltech,airi}
\end{icmlauthorlist}

\icmlaffiliation{skoltech}{Skolkovo Institute of Science and Technology}
\icmlaffiliation{airi}{Artificial Intelligence Research Institute}

\icmlcorrespondingauthor{Nikita Gushchin}{n.gushchin@skoltech.ru}
\icmlcorrespondingauthor{Sergei Kholkin}{s.kholkin@skoltech.ru}

\icmlkeywords{Machine Learning, ICML}

\vskip 0.3in
]



\printAffiliationsAndNotice{\icmlEqualContribution} % otherwise use the standard text.

\begin{abstract}
Schrödinger Bridges (SB) have recently gained the attention of the ML community as a promising extension of classic diffusion models which is also interconnected to the Entropic Optimal Transport (EOT). 
Recent solvers for SB exploit the pervasive bridge matching procedures. Such procedures aim to recover a stochastic process transporting the mass between distributions given only a transport plan between them. In particular, given the EOT plan, these procedures can be adapted to solve SB. This fact is heavily exploited by recent works giving rise to matching-based SB solvers. The cornerstone here is recovering the EOT plan: recent works either use heuristical approximations (e.g., the minibatch OT) or establish iterative matching procedures which by the design accumulate the error during the training. We address these limitations and propose a novel procedure to learn SB which we call the \textbf{optimal Schrödinger bridge matching}.
It exploits the optimal parameterization of the diffusion process and provably recovers the SB process \textbf{(a)} with a single bridge matching step and \textbf{(b)} with arbitrary transport plan as the input. Furthermore, we show that the optimal bridge matching objective coincides with the recently discovered energy-based modeling (EBM) objectives to learn EOT/SB. Inspired by this observation, we develop a light solver (which we call LightSB-M) to implement optimal matching in practice using the Gaussian mixture parameterization of the adjusted Schrödinger potential. We experimentally showcase the performance of our solver in a range of practical tasks. The code for our solver can be found at \url{https://github.com/SKholkin/LightSB-Matching}.

\end{abstract}

\begin{figure}
    \vspace{-2mm}
    \centering
        \includegraphics[width=0.92\linewidth]{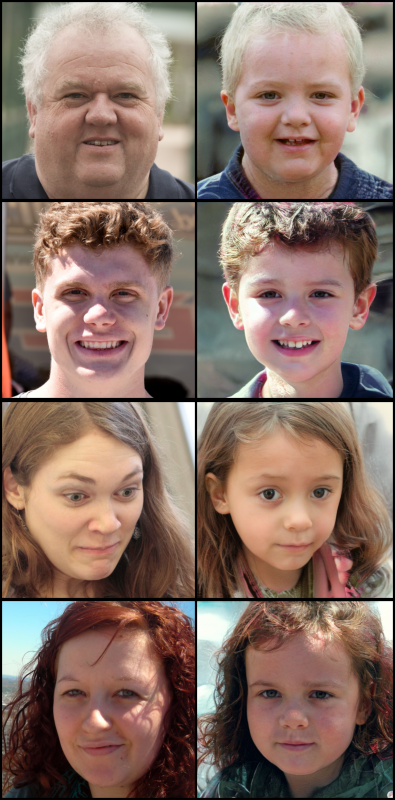}
    \vspace{0.5mm}
    \caption{Unpaired \textit{adult} $\rightarrow$ \textit{child} translation with our LightSB-M solver applied in the latent space of ALAE \citep{pidhorskyi2020adversarial} for 1024x1024 FFHQ images \citep{karras2019style}. \textit{Our LightSB-M solver converges on 4 cpu cores in several minutes}.}
    \label{fig:teaser}
\end{figure}

\vspace{0mm}
\section{Introduction}
\vspace{0mm}

Diffusion models are a powerful type of generative models that show an impressive quality of image generation \citep{ho2020denoising, rombach2022high}. However, they still have several directions for improvement on which the research community is actively working. Some of these directions are: speeding up the generation \citep{wang2022diffusion, song2023consistency}, application to the image-to-image transfer \citep{liu20232} extension to unpaired image transfer \citep{meng2021sdedit} or domain adaptation \citep{vargas2021solving}, including biological tasks with single cell data. 

A promising approach to advance these directions is the development of new theoretical frameworks for learning flows and diffusions. Recently proposed novel techniques such as flow \citep{lipman2022flow} and bridge \citep{shi2023diffusion} matching for flow and diffusion-based models show promising potential for further extending and improving generative and translation models. Furthermore, by exploiting theoretical links between flow and diffusion models with Optimal Transport \citep[OT]{villani2008optimal} and Schrödinger Bridge \citep[SB]{leonard2013survey} problems, several new methods have been proposed to speed up the inference \cite{liu2023instaflow}, to improve the quality of image generation \cite{liu20232}, and to solve unpaired image and domain translations \citep{de2021diffusion, shi2023diffusion}. 

Recent approaches \citep{tong2023simulation, shi2023diffusion,liu2022flow} to OT and SB problems based on flow and bridge matching either use iterative bridge matching-based procedures or employ heuristic approximations (e.g., the minibatch OT) to recover the SB through its relation to the Entropic OT problem. Unfortunately, iterative methods imply solving a sequence of time-consuming optimization problems and experience error accumulation. In turn, minibatch OT approximations can lead to biased solutions.

\textbf{Contributions}. We show that the above-mentioned issues can be eliminated. We do this by proposing a novel bridge matching-based approach to solve the SB in one iteration. 
\begin{enumerate}
  \item We propose a new bridge matching-based approach to solve the SB problem. Our approach exploits the novel "optimal" projection for stochastic processes that projects directly onto the set of SBs (\wasyparagraph\ref{sec:theory-optimal-sb-matching}).
  \item Based on the new theoretical results, we develop a new fast solver for the SB problem. We use the "light" parameterization for SBs \citep{korotin2024light} and our new theory on "optimal" projections to solve the SB problem in one bridge matching iteration (\wasyparagraph\ref{sec:light-optimization-procedure}).
  \item We perform extensive comparisons of this new solver on many setups where SB solvers are widely used, including the SB benchmark (\wasyparagraph{\ref{sec:exp-benchmark}}), single-cell data (\wasyparagraph{\ref{sec:exp-single-cell}}) and unpaired image translation (\wasyparagraph{\ref{sec:exp-image}}). 
\end{enumerate}
\textbf{Notations.} The notations of our paper mostly follow those used by the LightSB's authors in their work \citep{korotin2024light}. We work in $\mathbb{R}^{D}$, which is the $D$-dimensional Euclidean space equipped with the Euclidean norm $\|\cdot\|$. We use $\mathcal{P}(\mathbb{R}^{D})$ to denote the absolutely continuous Borel probability distributions whose variance and differential entropy are finite. To denote the density of $p\in\mathcal{P}(\mathbb{R}^{D})$ at a point $x\in\mathbb{R}^{D}$, we use $p(x)$. We use $\mathcal{N}(x|\mu,\Sigma)$ to denote the density at a point $x\in\mathbb{R}^{D}$ of the normal distribution with mean $\mu\in\mathbb{R}^{D}$ and covariance $0\prec \Sigma\in \mathbb{R}^{D\times D}$. We write $\KL{\cdot}{\cdot}$ to denote the Kullback-Leibler divergence between two distributions. In turn, $H(\cdot)$ denotes the differential entropy of a distribution. We use $\Omega$ to denote the space of trajectories, i.e., continuous $\mathbb{R}^{D}$-valued functions of $t\in [0,1]$. We write ${\mathcal{P}(\Omega)}$ to denote the probability distributions on the trajectories $\Omega$ whose marginals at $t=0$ and $t=1$ belong to $\mathcal{P}(\mathbb{R}^{D})$; this is the set of stochastic processes. We use $dW_{t}$ to denote the differential of the standard Wiener process $W\in\mathcal{P}(\Omega)$. For a process $T\in\mathcal{P}(\Omega)$, we denote its joint distribution at $t=0,1$ by $\pi^{T}\in\mathcal{P}(\mathbb{R}^{D}\times\mathbb{R}^{D})$. In turn, we use $T_{|x_0,x_1}$ to denote the distribution of $T$ for $t\in(0,1)$ conditioned on $T$'s values $x_0,x_1$ at $t=0,1$.

\vspace{-1mm}
\section{Preliminaries} \label{sec-background}
\vspace{-1mm}

We start with recalling the main concepts of the Schrödinger Bridge problem (\wasyparagraph\ref{sec-background-sb}). Next, we discuss the SB solvers which are the most relevant to our study (\wasyparagraph\ref{sec-background-lightsb}, \ref{sec-background-matching}).

\vspace{-1mm}
\subsection{Background on Schrödinger Bridges}\label{sec-background-sb}
\vspace{-1mm}

To begin with, we recall the SB problem with the Wiener prior and its equivalent Entropic Optimal Transport problem with the quadratic cost. We start from the latter as it is easier to introduce and interpret. For a detailed discussion of both these problems, we refer to \citep{leonard2013survey,chen2016relation}. Next, we describe the computational setup for learning SBs, which we consider in the paper.

\textbf{Entropic Optimal Transport (EOT) with the quadratic cost.} Consider distributions ${p_0\in\mathcal{P}(\mathbb{R}^{D})}$, ${p_1\in\mathcal{P}(\mathbb{R}^{D})}$. For $\epsilon>0$, the EOT problem with the quadratic cost is to find the minimizer of
\begin{equation}\label{entropic-ot}
\!\!\!\! \min_{\pi \in \Pi(p_0, p_1)} \int\limits_{\mathbb{R}^{D}} \! \int\limits_{\mathbb{R}^{D}} \! \frac{\|x_0-x_1\|^{2}}{2} \pi(x_0,x_1)dx_0dx_1 \! -\epsilon H(\pi), \!
\end{equation}
where $\Pi(p_0,p_1)$ is the set of the transport plans, i.e., probability distributions on $\mathbb{R}^{D}\times\mathbb{R}^{D}$ whose marginals are $p_0$ and $p_1$, respectively. The minimizer $\pi^{*}$ of \eqref{entropic-ot} exists, is unique, and is absolutely continuous; it is called the \textit{EOT plan}.

% Intuitively, solving \eqref{entropic-ot} one aims to find the translocation $\pi$ the probability mass of $p_0$ to $p_1$ with the minimal cost to transport the mass. The cost of moving each mass piece at a point $x_0$ to point $x_1$ is measured by the quadratic function $\frac{1}{2}\|x_0-x_1\|^{2}$. The total transport cost is given by the first term in \eqref{entropic-ot}, while the second entropy term $-\epsilon H(\pi)$ in \eqref{entropic-ot} is a regularization \citep{cuturi2013sinkhorn}. This second term improves the theoretical properties (provides the strict convexity in $\pi$, yield uniqueness of the solution $\pi^{*}$) of the problem and opens prospects for establishing efficient computational algorithms, see \citep{peyre2019computational} for a discussion. In the literature, problem \eqref{entropic-ot} sometimes appears in different but equivalent forms with KL divergence or conditional entropy, see \citep[Eq. 3-5]{gushchin2023building} and discussions therein.

\textbf{Schrödinger Bridge with the Wiener Prior.} Consider the Wiener process $W^{\epsilon}\in\mathcal{P}(\Omega)$ with volatility $\epsilon>0$ which starts at $p_0$ at $t=0$. Its differential satisfies the stochastic differential equation (SDE): $dW^{\epsilon}_{t}=\sqrt{\epsilon}dW_{t}$. The SB problem with the Wiener prior $W^{\epsilon}$ between $p_0,p_1$ is
\begin{equation}\label{sb-wiener}
\min_{T \in \mathcal{F}(p_0, p_1)}\KL{T}{W^{\epsilon}},
\end{equation}
where $\mathcal{F}(p_0, p_1)\subset \mathcal{P}(\Omega)$ is the subset of stochastic processes which start at distribution $p_0$ (at $t=0$) and end at $p_1$ (at $t=1$). There exists a unique minimizer $T^{*}$. Furthermore, it is a diffusion process described by the SDE: $dx_{t}=g^{*}(x_t,t)dt+dW^{\epsilon}_{t}$ \citep[Prop. 2.3]{leonard2013survey}. The optimal process $T^{*}$ is called the Schrödinger Bridge and $g^{*}:\mathbb{R}^{D}\times [0,1]\rightarrow \mathbb{R}^{D}$ is the optimal drift.

%Intuitively, solving 

\textbf{Relation of EOT and SB}. EOT \eqref{entropic-ot} and SB \eqref{sb-wiener} are closely related to each other. It holds that the joint marginal distribution $\pi^{T^{*}}$ of $T^{*}$ at times $0,1$ coincides with the EOT plan $\pi^{*}$ solving \eqref{entropic-ot}, i.e., $\pi^{T^{*}}=\pi^{*}$. Hence, the solution $\pi^{*}$ of the EOT problem \eqref{entropic-ot} can be recovered from $T^{*}$. Thus, SB can be viewed as a dynamic extension of EOT: a user is interested not only in the optimal mass transport plan $\pi^{*}$, but in the entire time-dependent mass transport process $T^{*}$.

Given just the optimal plan $\pi^{*}$, one may also complete it to get the full process $T^{*}$. It suffices to consider a process whose join marginal distribution at $t=0,1$ is $\pi^{*}$ and the trajectory distribution $T^{*}_{|x_0,x_1}$ at $t\in (0,1)$ conditioned on the ends $(x_0,x_1)$ coincides with the Wiener Prior's, i.e., ${T^{*}_{|x_0,x_1}=W_{|x_0,x_1}^{\epsilon}}$. The latter is known as the Brownian Bridge \citep[Sec. 8.3.3]{PINSKY2011391}. Thus, we obtain $T^* = \int_{\mathbb{R}^D\times\mathbb{R}^D} W^{\epsilon}_{|x_0, x_1} d\pi^*(x_0, x_1)$. This strategy does not directly give the optimal drift $g^{*}$, but it can recovered by other means, e.g., with the bridge matching (\wasyparagraph\ref{sec-background-matching}).

\textbf{Characterization for EOT and SB solutions.} It is known that the EOT plan $\pi^{*}$ can be represented through the input density $p_0$ and a function $v^{*}:\mathbb{R}^{D}\rightarrow\mathbb{R}_{+}$:
\begin{equation}\pi^{*}(x_0,x_1)=\underbrace{p_0(x_0)}_{=\pi^{*}(x_0)}\cdot \underbrace{\sfrac{\exp{\big(\sfrac{\langle x_0, x_1\rangle}{\epsilon}\big)v^{*}(x_1)}}{c_{v^{*}}(x_0)}}_{=\pi^{*}(x_1|x_0)},
\label{eot-plan-characterization}
\end{equation}
where $c_{v^{*}}(x_0)\defeq\int_{\mathbb{R}^{D}}\exp\big(\sfrac{\langle x_0,x_1\rangle}{\epsilon}\big)v^{*}(x_1)dy$. Following the notation of \cite{korotin2024light}, we call $v^{*}$ the adjusted Schrödinger potential. The optimal drift of $T^{*}$ can also be expressed using $v^{*}$. Namely,
\begin{eqnarray}
    g^*(x_t, t) =  \epsilon \nabla_{x_t} \log \Big(\int_{\mathbb{R}^{D}} \mathcal{N}(x'|x_t, (1-t)\epsilon I_{D})
    \nonumber
    \\
    \exp\big(\frac{\|x'\|^{2}}{2\epsilon}\big) v^*(x')dx'\Big) ,\label{sb-drift}\end{eqnarray}
see \citep[\wasyparagraph 2, 3]{korotin2024light} for a deeper discussion. Note that $v^{*}$ is defined up to the multiplicative constant.
 
\textbf{Computational SB/EOT setup}. In practice, distributions $p_0$ and $p_1$ are usually not available explicitly but only through their empirical samples ${\{x_0^{1},\dots,x_0^{N}\}\sim p_0}$ and ${\{x_{1}^{1},\dots,x_{1}^{M}\}\sim p_1}$. The typical task is to obtain a good approximation $\widehat{g}\approx g^{*}$ of the drift of SB process $T^{*}$ or explicitly/implicitly approximate the EOT plan's conditional distributions $\widehat{\pi}(\cdot|x_0)\approx \pi^{*}(\cdot|x_0)$ for \textit{all} $x_0\in\mathbb{R}^{D}$. This is needed to do the \textit{out-of-sample estimation}, i.e., for new (test) points $x_0^{new}\sim p_0$ sample ${x_1\sim \pi^{*}(\cdot|x_{0}^{new})}$ or simulate $T^{*}$'s trajectories staring at a point $x_0^{new}$ at time $t=0$. This setup widely appears in generative modeling \cite{de2021diffusion,gushchin2023entropic} and analysis of biological single cell data \cite{vargas2021solving,koshizuka2022neural,tong2023simulation}. 

The setup above is usually called the \textit{continuous} EOT or SB and should not be confused with the \textit{discrete} setup, which is widely studied in the discrete OT literature \citep{peyre2019computational,cuturi2013sinkhorn}. There one is mostly interested in computing the EOT plan directly between the empirical samples (probably weighted), i.e., match them with each other. There is usually no need in the out-of-sample estimation.

\vspace{-1mm}
\subsection{Energy-based EOT/SB Solvers} \label{sec-background-lightsb}
\vspace{-1mm}

Given a good approximation of the optimal potential $v^{*}$, one may approximate the conditional EOT plans and the optimal drift via \eqref{eot-plan-characterization} and \eqref{sb-drift}, respectively (using $v^{*}$'s approximation). Inspired by the idea above, papers \citep{korotin2024light,mokrov2024energyguided} provide related approaches to learn this potential. They show that $v^{*}$ can be learned via solving
% Thanks to \eqref{eot-plan-characterization}, one has ${\pi^{*}(x_1|x_0\!=\!0)\propto v^{*}(y)}$, i.e., $v^{*}$ is an unnormalized density of some distribution. 
% The authors propose to use the \textit{unnormalized} Gaussian mixture $v_{\theta}(y)=\sum_{k=1}^{K}\alpha_{k}\mathcal{N}(y|r_{k},\epsilon S_{k})$ to approximate this quantity. Here $\theta\defeq\{\alpha_{k},r_{k},S_{k}\}_{k=1}^{K}$ are the parameters ($\alpha_{k}> 0$, $S_{k}\succ 0$). They proved that such an approximation is universal (when $p,q$ are compactly supported) and parameters 
% theta can be trained via minimizing:
\begin{eqnarray}
\mathcal{L}_{0}(v)\defeq \min_{v} \Big\{\int_{\mathbb{R}^{D}}\log c_{v}(x_0)p_0(x_0)dx_0
\nonumber
\\
 -\int_{\mathbb{R}^{D}}\log v(x_1)p_1(x_1)dy \Big\},
\label{light-sb-objective}
\end{eqnarray}
where $c_{v}(x)\defeq\int_{\mathbb{R}^{D}}\exp\big(\sfrac{\langle x,y\rangle}{\epsilon}\big)v(y)dy$. This objective magically turns to be equal up to an additive $v$-independent constant to $\KL{\pi^{*}}{\pi_{v}}=\KL{T^{*}}{S_{v}}$, where
\begin{equation}
\pi_{v}(x_0,x_1) \defeq p_0(x_0)\underbrace{\frac{\exp\big(\sfrac{\langle x_0,x_1\rangle}{\epsilon}\big)v(x_1)}{c_{v}(x_0)}}_{=\pi_{v}(x_1|x_0)},
\label{plan-parametric-full}
\end{equation}
is an approximation of the optimal plan constructed by $v$ instead of $v^{*}$ in \eqref{eot-plan-characterization}. In turn, $S_{v}\in\mathcal{P}(\Omega)$ is a process with joint marginal (at $t=0,1$) is $\pi_{v}$ and $S_{|x_0,x_1}=W^\epsilon_{|x_0,x_1}$. Its drift $g_{v}$ can be recovered by using \eqref{sb-drift} with $v$ instead of $v^{*}$. 

Here we use the letter $S$ instead of $T$ to denote the process, and this is for a reason. With mild assumptions on $v$, the process $S_{v}$ is the Schrödinger bridge between $p_0$ and $p_{v}(x_1)\defeq \int_{\mathbb{R}^{D}}\pi_{v}(x_0,x_1)dx_0$, i.e., its marginal at $t=1$. This follows from the EOT benchmark constructor theorem \citep[Theorem 3.2]{gushchin2023building}. Hence, minimization \eqref{light-sb-objective} can be viewed as the optimization over processes $S_v$, which are SBs determined by their potential $v$.

Unfortunately, the optimization of \eqref{light-sb-objective} is tricky. While the potential $v$ can be directly parameterized, e.g., with a neural network $v_{\theta}$, the key challenge is to compute $c_{v}$, which is a non-trivial integral. Note that due to \eqref{eot-plan-characterization}, one has ${\pi^{*}(x_1|x_0\!=\!0)\propto v^{*}(y)}$, i.e., $v^{*}$ is an unnormalized density of some distribution. This fact is exploited in \cite{mokrov2024energyguided,korotin2024light} to establish ways to optimize \eqref{light-sb-objective}.

\textbf{Energy-guided EOT solver (EgNOT).} In \citep{mokrov2024energyguided}, the authors find out that, informally, objective \eqref{light-sb-objective} aims to find an unnormalized density $v^{*}$ by optimizing KL divergence. Therefore, it resembles the objectives of \textit{Energy-based Models} \citep[EBM]{lecun2006tutorial}. Inspired by this discovery, the authors show how the standard EBM approaches can be modified to optimize \eqref{light-sb-objective} and later sample from the learned plan $\pi_{v}$. The limitation of the approach is the necessity to use time-consuming MCMC techniques.

\textbf{Light Schrödinger Bridge solver (LightSB).} In \cite{korotin2024light}, they use the fact from \citep{gushchin2023building} that the Gaussian parameterization
\begin{equation}
    v_{\theta}(x_1)=\sum_{k=1}^{K}\alpha_{k}\mathcal{N}(x_1|\mu_{k},\epsilon \Sigma_{k})
    \label{gaussian-v}
\end{equation}
for $v$ provides a closed form analytic expression for $c_{\theta}$. This removes the necessity to use time-consuming MCMC approaches at both the training and the inference. Furthermore, Gaussian parameterization provides \textit{the closed form} expression for the drift of $S_{v}$  and allows lightspeed sampling from conditional distributions $\pi_{v}(x_{1}|x_{0})$, see \citep[Propositions 3.2, 3.3]{korotin2024light}.

\vspace{-1mm}
\subsection{Bridge matching Procedures for EOT/SB} \label{sec-background-matching}
\vspace{-1mm}

\textbf{Recovering SB process from EOT plan (OT-CFM).}
Since every SB solution is given by the EOT plan $\pi^*$ and the Brownian Bridges $W^{\epsilon}_{|x_0,x_1}$, i.e., ${T^{*} = \int_{\mathbb{R}^{D} \times \mathbb{R}^{D}} W^{\epsilon}_{|x_0,x_1} d\pi^*(x_0,x_1)}$, solution of the EOT problem $\pi^*$ already provides a way to sample from marginal distributions $p_{T^{*}}(x_t, t)$ of $T^{*}$ at each time $t\in [0,1]$. The authors of \citep{tong2023simulation} propose to use this property to recover the drift $g^*(x_t, t)$ of the process $T^{*}$ using flow \cite{lipman2022flow} and score matching techniques. They use the flow matching to fit the drift $g^{\circ}(x_t, t)$ of the probability flow ODE for marginals $p_{T^{*}}(x_t, t)$ (at time $t$) of the process $T^{*}$, i.e., $g^{\circ}(x_t, t)$ for which the continuity equation ${\frac{\partial p_{T^{*}}(x_t, t)}{\partial t} = -\nabla \cdot (p_{T^{*}}(x_t, t)g^{\circ}(x_t, t))}$ holds. In turn, score matching is used to fit the score functions $\nabla \log p_{T^{*}}(x_t, t)$ of marginal distributions. Then they recover the Schrödinger bridge drift by using the relationship between the probability flow ODE and the SDE representation of stochastic processes: $g^{\circ}(x_t, t) + \frac{\epsilon}{2}\nabla \log p_{T^{*}}(x_t, t) = g^*(x_t, t)$.

Unfortunately, the solution of the EOT problem $\pi^*$ for two arbitrary distributions $p_0$ and $p_1$ is unknown. The authors use the discrete (minibatch) OT between empirical distributions $\widehat{p}_0\defeq \sum_{n=1}^{N}\delta_{x_{n}}$ and $\widehat{p}_1\defeq \sum_{m=1}^{M}\delta_{y_{m}}$ constructed by available samples instead. However, the empirical EOT plan $\widehat{\pi}$ may be highly biased from the true $\pi^{*}$. This potentially leads to undesirable errors in approximating SB.

\textbf{Learning SB process without EOT solution (DSBM).}
Another matching method has been proposed by \citep{shi2023diffusion} to get the SB without knowing the EOT plan $\pi^*$. To begin with, for any $\pi\in\Pi(p_0,p_1)$, define $T_{\pi}$ (called the \textit{reciprocal} process of $\pi$) as a mixture of Brownian Bridges with weights given by $\pi$, i.e., $T_{\pi} = \int_{\mathbb{R}^{D} \times \mathbb{R}^{D}} W^{\epsilon}_{|x_0,x_1} d\pi(x_0,x_1)$.

To get $\pi^{*}$ and $T^{*}$, the authors alternate between two projections of stochastic processes: the \textit{reciprocal} and the \textit{Markovian}. For a process $T\in\mathcal{P}(\Omega)$, its reciprocal projection is a mixture of Brownian bridges given by the plan $\pi^{T}$:
\begin{equation}
\text{proj}_{\mathcal{R}}(T) \defeq \int_{\mathbb{R}^{D} \times \mathbb{R}^{D}} W^{\epsilon}_{|x_0,x_1} d\pi^{T}(x,y).
\label{eq:reciprocal-proj}
\end{equation}
This is a reciprocal process with the same joint marginal $\pi^{T}$ at times $t=0,1$ as $T$  (one may write $\text{proj}_{\mathcal{R}}(T)=T_{\pi^{T}}$).

Consider any reciprocal process $T_{\pi}$. Its Markovian projection $\text{proj}_{\mathcal{M}}(T_{\pi})$ is a diffusion process defined by an SDE $dx_t = g(x_t, t)dt + \sqrt{\epsilon}dW_t$, that preserves all time marginals of $T_{\pi}$. Its drift function is analytically given by: 
\begin{gather}\label{eq:markovian-proj-drift}
    % g(x_t, t) = \mathbb{E}_{(x_t, x_1) \sim T_{\pi}}[\frac{x_1 - x_t}{1 - t} |x_t = x].
    g(x_t, t) = \int_{\mathbb{R}^D} \frac{x_1 - x_t}{1-t} p_{T_{\pi}}(x_1|x_t) dx_1,
\end{gather}
where $p_{T_{\pi}}$ denotes the distribution of $T_{\pi}$. Drift \eqref{eq:markovian-proj-drift} is a solution to the following optimization problem:
\begin{gather}
    \!\!\!\!\! \min_{g} \int_{0}^1 \!\!\! \int_{\mathbb{R}^D \times \mathbb{R}^D} \!\! || g(x_t, t) - \frac{x_1 - x_t}{1-t}||^2 d p_{T_{\pi}}(x_t, x_1)dt
    \label{eq:optimal_proj_loss}
\end{gather}
and can be learned by sampling $(x_0, x_1) \sim \pi$, ${x_t \sim W^{\epsilon}_{|x_0, x_1}}$ and parametrizing $g$ by a neural network. This procedure is the so-called \textbf{bridge matching} procedure.

The authors prove \citep[Theorem 8]{shi2023diffusion} that a sequence $(T^{l})_{l \in \mathbb{N}}$ constructed by alternating  the projections
\begin{gather}
    T^{2l + 1} = \text{proj}_{\mathcal{M}}(T^{2l + 2}), \quad T^{2l} = \text{proj}_{\mathcal{R}}(T^{2l + 1}),
\end{gather}
with $T^0 = T_{\pi}$ and any $\pi \in \Pi(p_0, p_1)$ converges to the SB solution $T^{*}$ between $p_0$ and $p_1$. When $\epsilon \!\rightarrow\! 0$, the Markovian projection transforms into the well-known flow matching procedure \citep{lipman2022flow}, 
% for simulation-free learning of Continuous Normalizing Flows, 
and the whole iterative procedure becomes the Rectified Flow \citep{liu2022flow}. 

% Each process $T^n$ in this sequence has time marginals $p_0$ and $p_1$ at times $t=0$ and $t=1$.

Markovian projection \eqref{eq:markovian-proj-drift} is the bottleneck of the iterative procedure. In practice, the method uses a neural net to learn the drift of the projection. This introduces approximation errors at each iteration. The errors lead to differences between the process $T^n$'s marginal distribution at time $t=1$ and the actual $p_1$. These errors accumulate after each iteration and affect convergence, motivating the search for a bridge matching procedure that converges in a single iteration.

\vspace{-1mm}
\section{Light and Optimal SB Matching Solver}
\vspace{-1mm}

In \wasyparagraph\ref{sec:theory-optimal-sb-matching}, we present the main theoretical development of our paper -- the optimal Schrödinger bridge matching method. Next, in \wasyparagraph\ref{sec:light-optimization-procedure}, we propose our novel \textbf{LightSB-M} solver, which implements the method in practice. In \wasyparagraph\ref{sec-closely-related}, we discuss its connections with the related EOT/SB solvers. In Appendix~\ref{app:proofs} we \underline{provide proofs} of all theorems.

\vspace{-1mm}
\subsection{Theory. Optimal Schrödinger Bridge Matching}\label{sec:theory-optimal-sb-matching}
\vspace{-1mm}

% We denote the set of all Schrödinger Bridges starting in a given $p_0\in\mathcal{P}(\mathbb{R}^{D})$ and ending in some $p_1 \in \mathcal{P}(\mathbb{R}^{D})$ as $\mathcal{P}_{SB}$. Since every Schrödinger Bridge between some distributions $p_0$ and $p_1$ is characterized by an adjusted Schrödinger potential, we denote elements of $\mathcal{P}_{SB}$ as $S_{v}$, where $v$ is an adjusted Schrödinger potential (\wasyparagraph\ref{sec-background-sb}).
Our algorithm is based on the properties of KL projections of stochastic processes on the set $\mathcal{S}$ of Schrödinger Bridges:
\begin{gather}
\mathcal{S}\defeq\big\{S\in \mathcal{P}(\Omega)\text{ such that } \exists p_0^{S},p_1^{S}\in\mathcal{P}(\mathbb{R}^{D})
\nonumber
\\
\text{for which }S=\argmin_{T\in\mathcal{F}(p_0^{S},p_1^{S})}\KL{T}{W^{\epsilon}}\big\}.
\end{gather}

In addition to reciprocal and Markovian projections, we define a new "optimal projection" (OP). Consider \textbf{any} plan $\pi\in\Pi(p_0,p_1)$, e.g., independent, minibatch, optimal, etc. Given a \textbf{reciprocal} process $T_{\pi}$, its projection is the process
\begin{equation}
    \text{proj}_{\mathcal{S}}(T_{\pi}) \defeq \argmin_{S\in\mathcal{S}} \KL{T_{\pi}}{S}.
    \label{eq:optimal-projection}
\end{equation}

We prove that optimal projection allows to obtain the solution of Schrödinger Bridge in just one projection step.
\begin{theorem}[OP of a reciprocal process]\label{thm:optimal-projection}
The optimal projection of a reciprocal process $T_{\pi}$, given by a joint distribution $\pi \in \Pi(p_0, p_1)$ leads to the Schrödinger Bridge $T^{*}$ between the distributions $p_0$ and $p_1$, i.e.:
\begin{equation}\label{eq:optimal-projection-property}
    \text{proj}_{\mathcal{S}}(T_{\pi}) =  \argmin_{S\in\mathcal{S}} \KL{T_{\pi}}{S} =T^{*}.
\end{equation}
\end{theorem}
To implement this in practice, we need to \textbf{(a)} have a tractable estimator of $\KL{T_{\pi}}{S}$ and \textbf{(b)} be able to optimize over $\mathcal{S}$. We denote $\mathcal{S}(p_0)$ as the subset of $\mathcal{S}$ of processes which start at $p_0$ at $t=0$. Since $T^{*}\in \mathcal{S}(p_0)$, it suffices to optimize over $\mathcal{S}(p_0)$ in \eqref{eq:optimal-projection-property}. As it was noted in the background \wasyparagraph\ref{sec-background-lightsb}, processes $S\in\mathcal{S}(p_0)$ are determined by their adjusted Schrödinger potential $v$. We will write $S_{v}$ instead of $S$ for convenience.

% and parametrization of the drift $g(X_t, t)$ which will allow to optimize over the Schrödinger Bridges $S_g$. We start from the deriving the tractable objective for the optimal projection. 
\begin{theorem}[Tractable objective for the OP] \label{thm:tractable-objective-of-optimal-projection}
For the SB $S_v\in\mathcal{S}(p_0)$ and a reciprocal process $T_\pi$ with $\pi \in \Pi(p_0, p_1)$ the optimal projection objective \eqref{eq:optimal-projection} is
% \begin{equation}\label{eq:mse-optimal-projection}
%     \KL{T_{\pi}}{S_v} = C + \mathbb{E}_{M(T_{\pi})} ||v(X_t, t) -  g(X_t, t)||^2,
% \end{equation}
\begin{gather}
    \KL{T_{\pi}}{S_v} = C(\pi) + \label{eq:mse-optimal-projection}
    % \nonumber
    \\
    \frac{1}{2\epsilon}\int_{0}^{1}\!\! \int_{\mathbb{R}^D \times \mathbb{R}^D} ||g_v(x_t, t) -  \frac{x_{1} - x_{t}}{1 - t}||^2 dp_{T_{\pi}}(x_t, x_1) dt, 
    \nonumber
\end{gather}
where $g_{v}$ is the drift of $S_{v}$ given by \eqref{sb-drift} (with $v$ instead of $v^{*}$). Here the constant $C(\pi)$ does not depend on $S_{v}$.
\end{theorem}
This result provides an opportunity to optimize $S_{v}$ via fitting its drift $g_{v}$. Indeed, we can estimate $\KL{T_{\pi}}{S_{v}}$ up to a constant by sampling from $T_{\pi}$. To sample from $T_{\pi}$, it is sufficient to sample a pair $(x_0, x_1)\sim \pi$ and then to sample $x_t$ from the Brownian bridge $W^{\epsilon}_{|x_0, x_1}$. The natural remaining question is how to parameterize the drifts of the SB processes $S_{v}\in\mathcal{S}$. We explain this in the section below.

\vspace{-1mm}
\subsection{Practice. LightSB-M Optimization Procedure} \label{sec:light-optimization-procedure}
\vspace{-1mm}
To solve the Schrödinger Bridge between two distributions $p_0$ and $p_1$ by using optimal projection \eqref{eq:optimal-projection} and its tractable objective \eqref{eq:mse-optimal-projection}, we use \textbf{any} plan $\pi \in \Pi(p_0, p_1)$ accessible by samples. It can be the independent plan, i.e., just independent samples from $p_0, p_1$, any minibatch OT plan, i.e., the one obtained by solving discrete OT on minibatch from $p_0$ and $p_1$, etc. To optimize over Schrödinger Bridges $S_{v} \in \mathcal{S}$, we use the parametrization of $v$ as a Gaussian mixture \eqref{gaussian-v} from LightSB (\wasyparagraph{\ref{sec-background-lightsb}}), which for every $v_{\theta}$ provides $g_{\theta}\defeq g_{v_\theta}$ \eqref{sb-drift} in a closed form \citep[Proposition 3.3]{korotin2024light}:
\begin{gather}
    g_{\theta}(x, t) = \epsilon \nabla_{x} \log \big(\mathcal{N}(x|0, \epsilon(1-t)I_{D})
    \nonumber
    \\
    \sum_{k=1}^{K} \big\{\alpha_{k}\mathcal{N}(r_k|0, \epsilon \Sigma_k)\mathcal{N}(h(x, t)|0, A_k^t)\big\}\big)
    % \nonumber
    \label{sb-drift-gaussian}
\end{gather}
with $A_k^t \! \defeq \! \frac{t}{\epsilon(1-t)}I_{D} \! + \! \frac{\Sigma_k^{-1}}{\epsilon}$ and ${h_k(x, t) \! \defeq \! \frac{x}{\epsilon(1-t)} \! + \! \frac{1}{\epsilon}\Sigma_k^{-1}r_k}$. Using this parametrization and \textbf{any} $\pi \in \Pi(p_0, p_1)$, we optimize objective \eqref{eq:mse-optimal-projection} with the stochastic gradient descent.

\vspace{-1mm}
\begin{algorithm}[H]
    \caption{Light SB Matching (LightSB-M)}
    \label{alg:light-sbm}
    \SetKwInOut{Input}{Input}\SetKwInOut{Output}{Output}
    \Input{plan $\pi \in \Pi(p_0, p_1)$ accessible by samples; 
    adjusted Schrödinger potential $v_{\theta}$ parametrized by a gaussian mixture ($\theta = \{\alpha_k, \mu_k, \Sigma_k\}_{k=1}^K$).}
    \Output{learned drift $g_{\theta}$ approximating the optimal $g^{*}$.}
    
    \Repeat{converged}{
        Sample batch of pairs $\{x^n_0, x^n_1\}_{n=0}^N \sim \pi$; \\
        Sample batch $\{t_n\}_{n=0}^N \sim U[0, 1]$; \\
        Sample batch $\{x^n_t\}_{n=0}^N \sim W^{\epsilon}_{|x_0,x_1}$; \\
        $\mathcal{L}_{\theta} \leftarrow \frac{1}{N}\sum_{n=1}^N ||g_{\theta}(x^n_t, t_n) - \frac{1}{1-t_n}(x^n_1 - x^n_t)||^2$;\\
        Update $\theta$ using $\frac{\partial \mathcal{L}_{\theta}}{\partial \theta}$;
    }
\end{algorithm}
\vspace{-1mm}

The \textbf{training} procedure is described in Algorithm~\ref{alg:light-sbm}. We recall that the Brownian bridge $W^{\epsilon}_{|x_0,x_1}$ has time marginals $p_{BB}(x_t|x_0, x_1) \defeq \mathcal{N}(x_t|tx_1 + (1-t)x_0, \epsilon t(1-t)I_{D})$, i.e. has a normal distribution with a scalar covariance matrix.

After learning the drift $g_{v}(x, t)$ of the Schrödinger Bridge SDE $dx_t = g_{v}(x_t,t)dt + \sqrt{\epsilon}dW_t$, one can use any SDE solver to \textbf{infer} trajectories. For example, one can use the simplest and most popular Euler-Maruyama scheme \citep[\wasyparagraph{9.2}]{kloeden1992stochastic}. However, SDE solvers introduce some errors due to discrete approximations. Using the LightSB parameterization of Schrödinger bridges from \citep{korotin2024light}, we can sample trajectories without having to solve the learned SDE numerically. To do so, we first sample from the learned plan $\pi_{v}(x_1|x_0)$ given by \eqref{plan-parametric-full} and then sample the trajectory of the Brownian bridge $W^{\epsilon}_{|x_0,x_1}$ using it's self-similarity property \citep[\wasyparagraph{3.2}]{korotin2024light}. We recall that the self-similarity of the Brownian bridge means that if we have a trajectory $x_0, x_{t_1}, ..., x_{t_{L}}, x_1$, we can sample a new point at time $t_{l} < t < t_{l+1}$ by using the following property of the Brownian bridge:
\begin{eqnarray}
    x_{t}\!\sim\!\mathcal{N}\big(x_{t}|x_{t_l}\!\!+\!\frac{t'-t_{l}}{t_{l+1}\!\!-t_{l}}(x_{t_{l+1}}\!\!-x_{t_{l}}),\!\epsilon \frac{(t'\!-\!t_{l})(t_{l+1}\!-\!t')}{t_{l+1}-t_l}\big).
    \nonumber
\end{eqnarray}

\vspace{-1mm}
\subsection{Connections to the Most Related Prior Works} \label{sec-closely-related}
\vspace{-1mm}

\textbf{DSBM} \citep{shi2023diffusion}. Schrödinger Bridge $T^*$ between $p_0$ and $p_1$ is the only process that simultaneously is Markovian and reciprocal \citep[Proposition 2.3]{leonard2013survey}. This fact lies at the core of DSBM's iterative approach of alternating Markovian and reciprocal projections. In turn, our optimal projection \eqref{eq:optimal-projection} provides the SB in \textbf{one step}, projecting a process on the set of processes that are both reciprocal and Markovian, i.e., Schrödinger Bridges.

\textbf{OT-CFM} \citep{tong2023simulation}. Our optimal projection \eqref{eq:optimal-projection} of a reciprocal process of $T_{\pi}$ with \textbf{any} $\pi \in \Pi(p_0, p_1)$ is the same Schrödinger Bridge between $p_0$ and $p_1$. Thus, optimal projection does not depend on the choice of the plan $\pi$. In turn, OT-CFM provides theoretical guarantees of finding the Schrödinger Bridge only if one chooses as plan $\pi$ the EOT plan $\pi^*$, which is unknown for arbitrary distributions $p_0, p_1$. 
% \color{red} 
% write that one may use any plan and this is not a heuristic
% we can use any plan and this is well-justified
% \color{black}
\textbf{EgNOT/LightSB} \citep{mokrov2024energyguided,korotin2024light}. Our main objective \eqref{eq:optimal-projection} resembles objective \eqref{light-sb-objective} of EgNOT and LightSB as the latter equals  $\KL{T^*}{S_v}$ up to a constant. At the same time, our objective allows to use \textbf{any} reciprocal process $T_{\pi}$ instead of $T^*=T_{\pi^{*}}$.
% \color{red}
% write that our main objective resembles KL(T*|T) of EGNOT and lightSB. Theorem that they are equivalent
% \color{black}
Interestingly, our obtained tractable bridge matching objective turns out to be closely related to the EgNOT/LightSB objective \eqref{light-sb-objective}.
\begin{theorem}[Equivalence to EgNOT/LightSB objective] \label{thm:eqviv-lightsb}
The OP objective \eqref{eq:mse-optimal-projection} for a reciprocal process $T_{\pi}$ and $\pi \in \Pi(p_0, p_1)$ is equivalent to LightSB objective $\mathcal{L}_0$ \eqref{light-sb-objective}:
\begin{gather}
    \frac{1}{2\epsilon}\int_{0}^{1}\!\! \int_{\mathbb{R}^D \times \mathbb{R}^D} ||g_v(x_t, t) -  \frac{x_{1} - x_{t}}{1 - t}||^2 dp_{T_{\pi}}(x_t, x_1) dt = 
    \nonumber
    \\
    \widetilde{C}(\pi) + \mathcal{L}_0(v).
    \nonumber
\end{gather}
\end{theorem}
One interesting conclusion from this equivalence is that our LightSB-M solver automatically inherits the theoretical generalization and approximation properties of the LightSB solver; see \citep[\wasyparagraph 3]{korotin2024light} for details about them.

\vspace{-1mm}
\section{Other Related Works}
\vspace{-1mm}
Here, we overview other existing works related to solving SB/EOT. Unlike the works described above, these are less relevant to our study. Still, we want to highlight some aspects of other solvers related to our solver.

\vspace{-1mm}
\subsection{Iterative proportional fitting (IPF) solvers.} 
\vspace{-1mm}
There are several Schrödinger Bridge solvers \citep{vargas2021solving, de2021diffusion, chen2021likelihood} for continuous probability distributions based on the Iterative Proportional Fitting (IPF) procedure \citep{fortet1940resolution, kullback1968probability, ruschendorf1995convergence}. The IPF procedure is related to the Sinkhorn algorithm \citep{cuturi2013sinkhorn} and, as was recently shown in work \citep{vargas2023transport}, coincides with the expectation-maximization (EM) algorithm \citep{dempster1977maximum}. All these three IPF-based SB solvers consist of iterative reversing of Markovian processes and differ only in particular methods to fit a reversion of a process by a neural network. The first two \citep{vargas2021solving, de2021diffusion} methods use similar mean-matching procedures, while the last \citep{chen2021likelihood} utilizes a different approach which includes the estimation of a divergence. 

In \citep{shi2023diffusion} the authors show, that due to iterative nature of one of these solvers \citep{de2021diffusion} it can diverge, due to errors accumulation on each iteration. Furthermore, the authors of \citep{vargas2023transport} show that these solvers tend to lose the information of Wiener Prior of Schrödinger Bridge and converge to the Markovian process that does not solve the SB problem. In turn, \textit{our approach eliminates the need for iterative learning} of a sequence of Markovian processes and is free from the possible issues with divergence or obtaining a biased solution.

\begin{figure*}[!t]
% \vspace{-6mm}
\begin{subfigure}[b]{0.245\linewidth}
\centering
\includegraphics[width=0.995\linewidth]{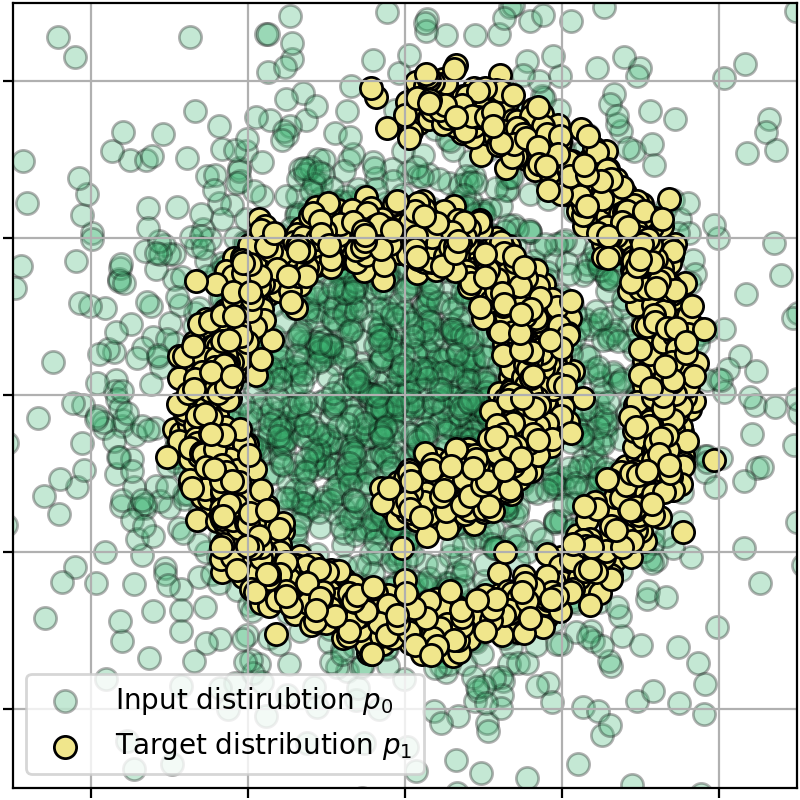}
\caption{\centering ${x\sim p_0
}$, ${y \sim p_1}.$}
\vspace{-1mm}
\end{subfigure}
\vspace{-1mm}\hfill\begin{subfigure}[b]{0.245\linewidth}
\centering
\includegraphics[width=0.995\linewidth]{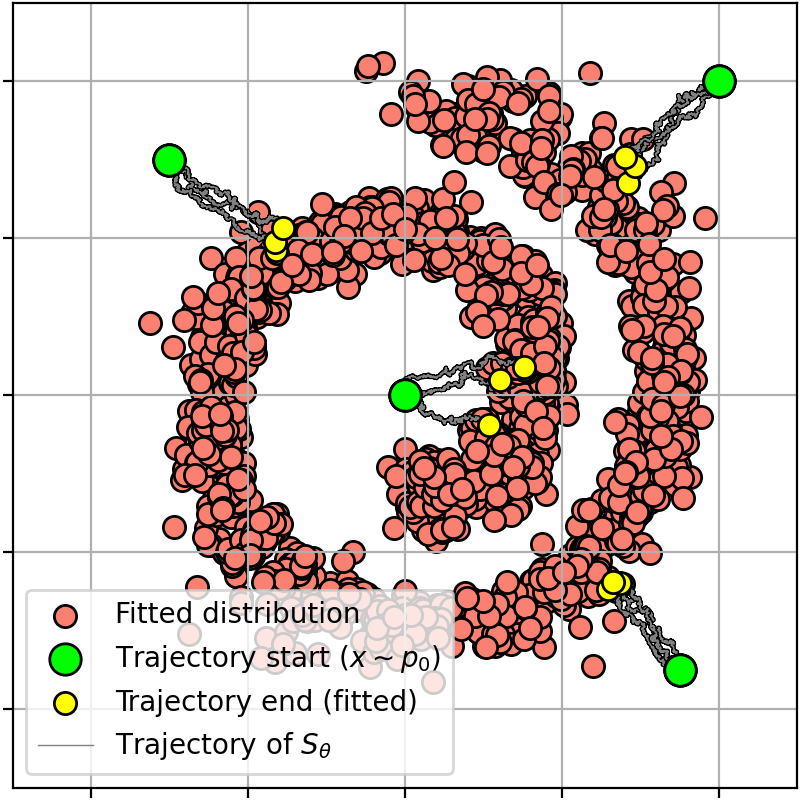}
\caption{\centering $\epsilon=0.01$.}
\vspace{-1mm}
\end{subfigure}
\hfill\begin{subfigure}[b]{0.245\linewidth}
\centering
\includegraphics[width=0.995\linewidth]{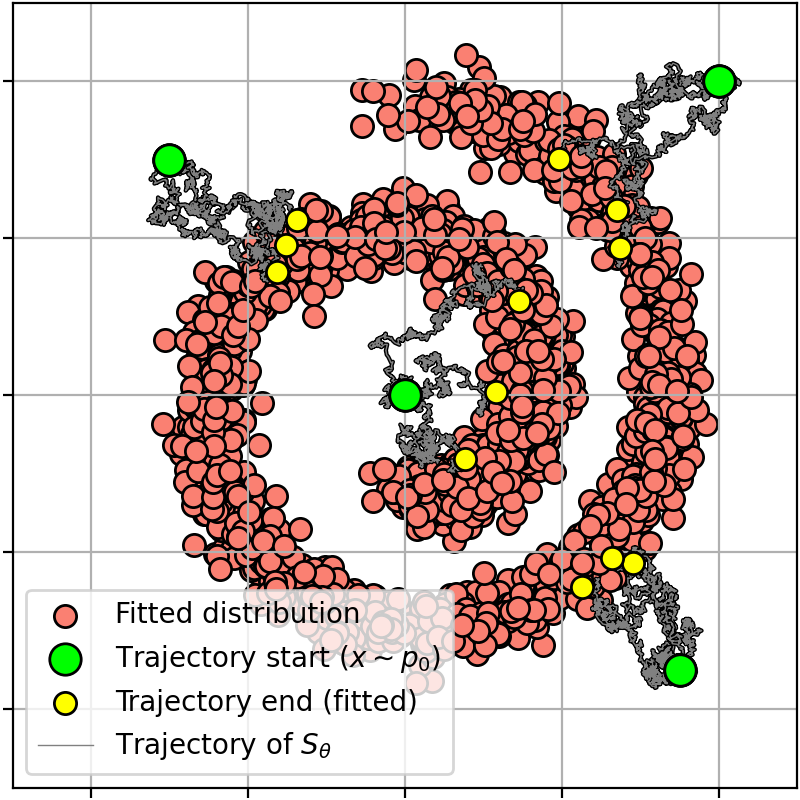}
\caption{\centering $\epsilon=0.1$.}
\vspace{-1mm}
\end{subfigure}
\hfill\begin{subfigure}[b]{0.245\linewidth}
\centering
\includegraphics[width=0.995\linewidth]{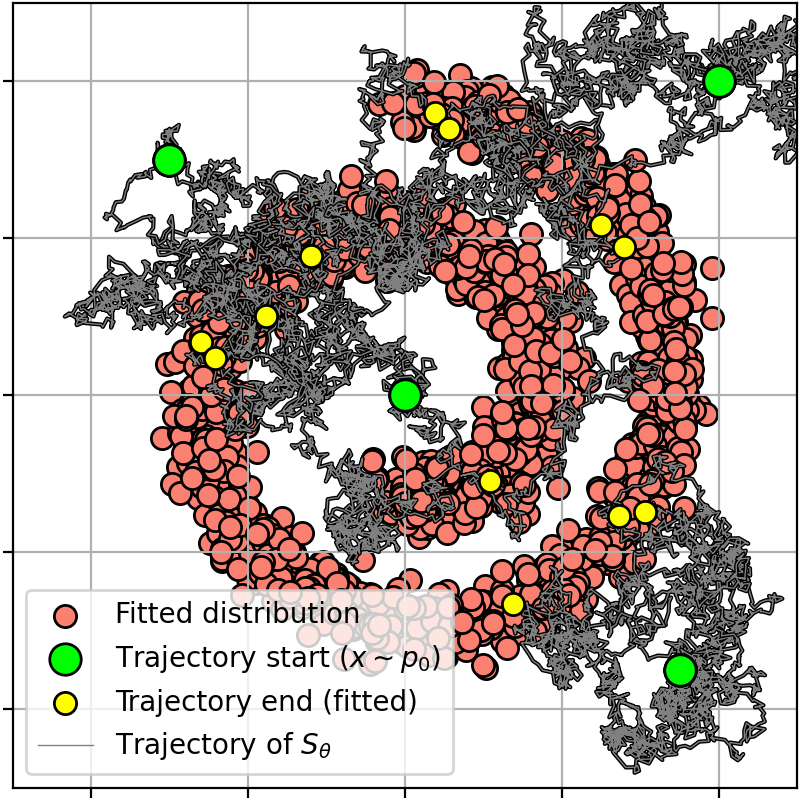}
\caption{\centering $\epsilon=1$.}
\vspace{-1mm}

\end{subfigure}
\vspace{-0mm} \caption{\centering The process $S_{\theta}$ learned with LightSB-M \textbf{(ours)} in \textit{Gaussian} $\!\rightarrow\!$ \textit{Swiss roll} example (\wasyparagraph\ref{sec:exp-2D}).}
\label{fig:swiss-roll}
\vspace{-4mm}
\end{figure*}

\vspace{-1mm}
\subsection{EOT solvers and EOT-based SB solvers.} 
\vspace{-1mm}
Recall that EOT and SB problems are closely related: SB solutions can be recovered from EOT solutions by using Brownian Bridge $W^{\epsilon}_{|x_0,x_1}$ or recovering the drift $g(x_t, t)$, e.g., as in \citep{tong2023simulation}. Due to this, we also give a quick overview of EOT solvers for continuous distributions. Several works \citep{genevay2016stochastic, seguy2018large, daniels2021score} consider solving the EOT problem by utilizing the classic dual EOT problem \citep{genevay2019sample}. Classic dual EOT problem for continuous $p_0$ and $p_1$ is an unconstrained maximization problem over dual variables, also called potentials, which can be parameterized by neural networks and trained. After training, these potentials can be used to directly sample from distribution $\pi^*(x_1|x_0)$ by using additional score model for $\nabla_x \log p_1(x)$ \citep{daniels2021score} or to train neural network model to predict conditional expectation $\mathbb{E}_{\pi^*(x_1|x_0)}x_1$, i.e., the barycentric projection. However, the main disadvantage of these methods is that in practice, dual EOT problem cannot be solved by neural networks for practically meaningful (small) coefficients $\epsilon$ due to numerical errors of calculating dual EOT objective since it includes terms in form $\mathbb{E}_{x_0 \sim p_0, x_1 \sim p_1}\exp(\frac{f(x_0, x_1)}{\epsilon})$. 

There is also one SB solver based on the theory of EOT dual problem \citep{gushchin2023entropic}. This solver directly fits the drift $g$ of the Schrödinger Bridge by using a maximin reformulation of the dual EOT problem and its link to the SB problem. This allows to overcome the numerical problems and solve SB for practically meaningful values of $\epsilon$. 

Our solver is also based on solving EOT and SB using the theory behind the dual EOT problem. Thanks to using parametrization of adjusted Schrödinger potential as in \citep{korotin2024light} instead of EOT potentials as in \citep{seguy2018large, daniels2021score} and using novel optimization objective based on bridge matching, \textit{our method overcomes numerical issues of the previously developed dual EOT-based methods} without the maximin optimization.

\vspace{-1mm}
\subsection{Other SB solvers.} 
\vspace{-1mm}
The authors of \citep{kim2024unpaired} propose a different minimax SB solver by considering the self-similarity of the SB in learning objectives and an additional consistency regularization. While showing good results, their approach requires using neural estimation of entropy, which involves solving additional optimization problem at every minimization step.

All previously considered solvers are designed to solve SB as a problem of finding the optimal translation between two distributions $p_0$, $p_1$ without any paired data from them, but there are also several SB solvers \citep{liu20232, somnath2023aligned} for setups with paired trained data such as the super-resolution. In fact, the concept of bridge matching was introduced in \cite{liu20232} but for the paired setup. The authors work under the assumption that the available paired data is a good approximation of the EOT plan and propose using Bridge matchi to recover the SB from this data, which makes their method related to \citep{tong2023simulation}. As noted earlier, \textit{our solver provably recovers SB using data provided by arbitrary plan $\pi$ between $p_0$ and $p_1$.}

\begin{table*}[h]
\centering
\scriptsize
\setlength{\tabcolsep}{3pt}
\begin{tabular}{rrcccccccccccc}
\toprule
& & \multicolumn{4}{c}{$\epsilon=0.1$} & \multicolumn{4}{c}{$\epsilon=1$} & \multicolumn{4}{c}{$\epsilon=10$}\\
\cmidrule(lr){3-6} \cmidrule(lr){7-10} \cmidrule(l){11-14}
& Solver Type & {$D\!=\!2$} & {$D\!=\!16$} & {$D\!=\!64$} & {$D\!=\!128$} & {$D\!=\!2$} & {$D\!=\!16$} & {$D\!=\!64$} & {$D\!=\!128$} & {$D\!=\!2$} & {$D\!=\!16$} & {$D\!=\!64$} & {$D\!=\!128$} \\
\midrule  
Best solver on benchmark$^\dagger$ & Varies &  $1.94$  &  $13.67$  &  $11.74$  &  $11.4$  &  $1.04$ &  $9.08$ &   $18.05$  &  $15.23$  &  $1.40$  &  $1.27$  &  $2.36$  &  $1.31$ \\
LightSB$^\dagger$ & KL minimization & $0.03$ & $0.08$ & $0.28$ & $0.60$ & $0.05$ & $0.09$ & $0.24$ & $0.62$ & $0.07$ & $0.11$ & $0.21$ & $0.37$ \\ 
\hline
DSBM & \multirow{5}{*}{Bridge matching} & $5.2$ & $16.8$ & $37.3$ & $35$ & $0.3$ & $1.1$ & $9.7$ & $31$ & $3.7$ & $105$ & $3557$ & $15000$ \\   
$\text{SF}^2$M-Sink & & $0.54$ & $3.7$ & $9.5$ & $10.9$ & $0.2$ & $1.1$ & $9$ & $23$ & $0.31$ & $4.9$ & $319$ & $819$ \\  
LightSB-M (ID, \textbf{ours}) & & $0.04$ & $0.18$ & $0.77$ & $1.66$ & $0.09$ & $\mathbf{0.18}$ & $0.47$ & $1.2$ & $\mathbf{0.12}$ & $0.19$ & $\mathbf{0.36}$ & $0.71$ \\  
LightSB-M (MB,  \textbf{ours}) & & $\mathbf{0.02}$ & $\mathbf{0.1}$ & $0.56$ & $1.32$ & $0.09$ & $\mathbf{0.18}$ & $\mathbf{0.46}$ & $\mathbf{1.2}$ & $0.13$ & $\mathbf{0.18}$ & $\mathbf{0.36}$ & $0.71$ \\   
LightSB-M (GT,  \textbf{ours}) & & $\mathbf{0.02}$ & $\mathbf{0.1}$ & $\mathbf{0.49}$ & $\mathbf{1.16}$ & $\mathbf{0.09}$ & $\mathbf{0.18}$ & $0.47$ & $\mathbf{1.2}$ & $0.13$ & $\mathbf{0.18}$ & $\mathbf{0.36}$ & $\mathbf{0.69}$ \\
\bottomrule
\end{tabular}
\vspace{-1mm}
\captionsetup{justification=centering, font=footnotesize}
\caption{Comparisons of $\text{cB}\mathbb{W}_{2}^{2}\text{-UVP}\downarrow$ (\%) between the optimal plan $\pi^*$ and the learned plan  $\pi_{\theta}$ on the EOT/SB benchmark (\wasyparagraph\ref{sec:exp-benchmark}). \\ The best metric over \textit{bridge matching} solvers is \textbf{bolded}. Results marked with $\dagger$ are taken from \cite{korotin2024light}.}
\label{table-cbwuvp-benchmark}
\vspace{-1mm}
\end{table*}

\color{black}
\begin{table*}[!t]
\vspace{1mm}
% \color{blue}
\centering
\scriptsize
\begin{tabular}{ |c|c|c|c|c| }
\hline
\textbf{Solver type} & \backslashbox{\textbf{Solver}}{\textbf{DIM}} & \textbf{50} & \textbf{100} & \textbf{1000} \\ 
\hline
 Langevin-based & \citep{mokrov2024energyguided}$^\dagger$ [1 GPU V100] & $2.39 \pm 0.06$ ($19$ m) & $2.32 \pm 0.15$ ($19$ m) & $1.46 \pm 0.20$ ($15$ m) \\
\hline
 Minimax & \citep{gushchin2023entropic}$^\dagger$ [1 GPU V100] & $2.44 \pm 0.13$ ($43$ m) & $2.24 \pm 0.13$ ($45$ m) & $1.32 \pm 0.06$ ($71$ m) \\
 \hline
 IPF & \citep{vargas2021solving}$^\dagger$ [1 GPU V100] & $3.14 \pm 0.27$ ($8$ m) & $2.86 \pm 0.26$ ($8$ m) & $2.05 \pm 0.19$ ($11$ m) \\
\hline
 KL minimization & LightSB \citep{korotin2024light}$^\dagger$ [4 CPU cores] & $2.31 \pm 0.27$ ($65$ s) & $2.16 \pm 0.26$ ($66$ s) & $1.27 \pm 0.19$ ($146$ s) \\
 \hline
 \multirow{4}{*}{Bridge matching} & DSBM  \citep{shi2023diffusion} [1 GPU V100] & $2.46 \pm 0.1$ ($6.6$ m) & $2.35 \pm 0.1$ ($6.6$ m) & $1.36 \pm 0.04$ ($8.9$ m) \\
 \cline{2-5}
 & $\text{SF}^2$M-Sink \citep{tong2023simulation} [1 GPU V100] & $2.66 \pm 0.18$ ($8.4$ m) & $2.52 \pm 0.17$ ($8.4$ m) & $1.38 \pm 0.05$ ($13.8$ m) \\
 \cline{2-5} 
 & LightSB-M (ID, \textbf{ours}) [4 CPU cores] & $2.347 \pm 0.11$ ($58$ s) & $2.174 \pm 0.08$ ($60$ s) & $1.35 \pm 0.05$ ($147$ s) \\
 \cline{2-5} 
  & LightSB-M (MB, \textbf{ours}) [4 CPU cores] & $\mathbf{2.33} \pm 0.09$ ($80$ s) & $\mathbf{2.172} \pm 0.08$ ($80$ s) & $\mathbf{1.33} \pm 0.05$ ($176$ s) \\
\hline
\end{tabular}
\vspace{-1mm}
\captionsetup{justification=centering, font=footnotesize}
 \caption{Energy distance (averaged for two setups and 5 random seeds) on the MSCI dataset (\wasyparagraph\ref{sec:exp-single-cell}) along with $95\%$-confidence interval ($\pm$ intervals) and average training times (s - seconds, m - minutes). The best \textit{bridge matching} solver according to the mean value is \textbf{bolded}. Results marked with $\dagger$ are taken from \citep{korotin2024light}.}
 \label{table-sc-comparison}
\vspace{-5mm}
\end{table*}

\vspace{-1mm}
\section{Experimental Illustrations}\label{sec:experiments}
\vspace{-1mm}
To evaluate our new LightSB-M solver, we considered several setups from related works.
% , including an illustrative 2D setup (\wasyparagraph{\ref{sec:exp-2D}}), a benchmark for EOT/SB problems (\wasyparagraph{\ref{sec:exp-benchmark}}), single-cell biological data (\wasyparagraph{\ref{sec:exp-single-cell}}) and unpaired image-to-image translation (\wasyparagraph{\ref{sec:exp-image}}). 
The code for our solver is written in \texttt{PyTorch} and available at \url{https://github.com/SKholkin/LightSB-Matching}. For each experiment, we present a separate self-explaining Jupyter notebook, which can be used to reproduce the results of our solver. We provide the technical \underline{details} in Appendix~\ref{app:solver-details}.

\vspace{-1mm}
\subsection{Qualitative 2D Example} \label{sec:exp-2D}
\vspace{-1mm}
We start our evaluation with an illustrative 2D setup. We solve the SB between a Gaussian distribution $p_0$ and a Swiss roll $p_1$. We run our LightSB-M solver with mini-batch (MB) discrete OT as plan $\pi$ for different values of the coefficient $\epsilon$ and present the results in Figure~\ref{fig:swiss-roll}. As expected, we see that the amount of noise in the trajectories and the stochasticity of the learned map are proportional to coefficient $\epsilon$. The technical \underline{details} of this setup are given in Appendix~\ref{app:hyperparameters-for-2D}.

\vspace{-1mm}
\subsection{Quantitative Evaluation on the SB Benchmark}
\vspace{-1mm}
\label{sec:exp-benchmark}
We use the SB mixtures benchmark proposed by \citep[\wasyparagraph{4}]{gushchin2023building} to experimentally verify that our approach based on the optimal projection is indeed able to solve the Schrödinger Bridge between $p_0$ and $p_1$ \textit{by using any reciprocal process $T_{\pi}$, $\pi \in \Pi(p_0, p_1)$}. The benchmark provides continuous probability distribution pairs $p_0,p_1$ for dimensions $D \in \{2,16,64,128\}$ with the known EOT plan $\pi^*(x_0, x_1)$ for parameter $\epsilon \in \{0.1,1.10\}$. To evaluate the quality of the SB solution (EOT plan) we use $\text{cB}\mathbb{W}_{2}^{2}\text{-UVP}$ metric as suggested by the authors \citep[\wasyparagraph{5}]{gushchin2023building}. Additionally, we study how well the solvers restore the target distribution $p_1$ in Appendix~\ref{app:benchmark}.

We provide results of our LightSB-M solver with independent (ID) and mini-batch discrete OT (MB) as $\pi$ in $T_{\pi}$ for mixture benchmark pairs in Table~\ref{table-cbwuvp-benchmark}. Since the benchmark provides the ground truth EOT plan $\pi^{*}$ (GT), we also run our solver with it. Note that we have access to the GT EOT plan thanks to the benchmark, and in regular setups there is, of course, no access to it. As shown in the Table~\ref{table-cbwuvp-benchmark}, our solver demonstrates comparable performance to the best among other solvers for all considered plans $\pi$. As noted in \citep[\wasyparagraph{5.2}]{korotin2024light}, the mixture parameterization used by LightSB and which we adapt in our LightSB-M solver may introduce some inductive bias, since it uses the analogous principles used to construct the benchmark.

\vspace{-1mm}\hspace{-0.8mm}\fbox{%
    \parbox{0.985\linewidth}{
    \centering
    We empirically see that our LightSB-M solver finds the same (optimal) solution for all considered plans $\pi$.
    }
}

% Since we use the same parametrization as in LightSB solver \citep{korotin2024light} we provide results from their work including comprehensive study of other solver on the SB benchmark in first two rows of Table~\ref{table-cbwuvp-benchmark}. It shows, that our solver performs comparable
\textbf{Baselines.} We present results for other bridge matching methods such as DSBM \citep{shi2023diffusion}, which uses Markovian and reciprocal projections, and $\text{SF}^2$M-Sink \citep{tong2023simulation}, which uses an approximation of the EOT plan by the Sinkhorn algorithm \citep{cuturi2013sinkhorn}. On the setups with $\epsilon=10$ both methods exibits difficulties due to the necessity to learn SDE with high magnitude.  On the setups with $\epsilon=0.1$ and $\epsilon=1$, $\text{SF}^2$M-Sink works better than DSBM. This result may seem counterintuitive at first, since DSBM methods should find the true SB solution, while $\text{SF}^2$M-Sink should find some approximation to it based on how close the minibatch discrete EOT approximates the GT EOT plan. One possible reason is that DSBM simply requires more iterations of Markovian/reciprocal projections. However, in our experiments we observe that increasing the number of iterations does not improve the quality.

We provide an \underline{additional study} of dynamic metrics and the inference speed of our solver in Appendix~\ref{app:benchmark}.

% (see Appendix~\ref{app:benchmark-eval} for details).

% DSBM works relatively well (compared to the best solver from the benchmark paper marked by $\dagger$) on the setups with $\epsilon=1$, but exibits difficulties with $\epsilon=10$ due to the necessity to learn SDE with high magnitude of noise and $\epsilon=0.1$ possibly due to the necessity to use more iterations or accumulation of errors.  $\text{SF}^2$M-Sink also has difficulties with $\epsilon=10$ for the same reason. On $\epsilon=1$ it shows worse $\text{cB}\mathbb{W}_{2}^{2}\text{-UVP}$ than DSBM, possibly due to the approximation error of the plan used.

\vspace{-1mm}
\subsection{Quantitative Evaluation on Biological Data}
\vspace{-1mm}
\label{sec:exp-single-cell}
We evaluate our algorithm on the inference of cell trajectories from unpaired single-cell data problem, where OT/SB is widely used \citep{vargas2021solving, tong2023simulation, koshizuka2022neural}. We consider the recent high-dimensional single-cell setup provided by \cite{tong2023simulation} based on the dataset from the Kaggle competition "Open Problems - Multimodal Single-Cell Integration." This dataset provides single-cell data from four human donors on days $2,3,4$ and $7$ and describes the gene expression levels of distinct cells. The task of this setup is to learn a trajectory model for the cell dynamics, given only unpaired samples at two time points, representing distributions $p_0$ and $p_1$. As in related works \citep{tong2023simulation, korotin2024light}, we use PCA projections of the original data with $\text{DIM} \in \{50, 100, 1000\}$ components. 

In our experiments, we consider two setups by taking data from two different days as $p_0, p_1$ to solve the Shrödinger Bridge and one intermediate day for evaluation. The first setup includes data from day 2 as $p_0$, data from day $4$ as $p_1$, and data from day $3$ for evaluation, while the second setup includes data from day $3$ as $p_0$, data from day $7$ as $p_1$, and data from day $4$ for evaluation. At evaluation, we use learned models to sample one trajectory for each cell from the initial distribution $p_0$ and then compare the predicted distribution at the intermediate time point with the ground truth data distribution. For comparison, we use energy distance \citep{rizzo2016energy} and present results in Table~\ref{table-sc-comparison}.

\begin{figure*}[!t]
\begin{subfigure}[b]{0.49\linewidth}
\centering
\includegraphics[width=0.995\linewidth]{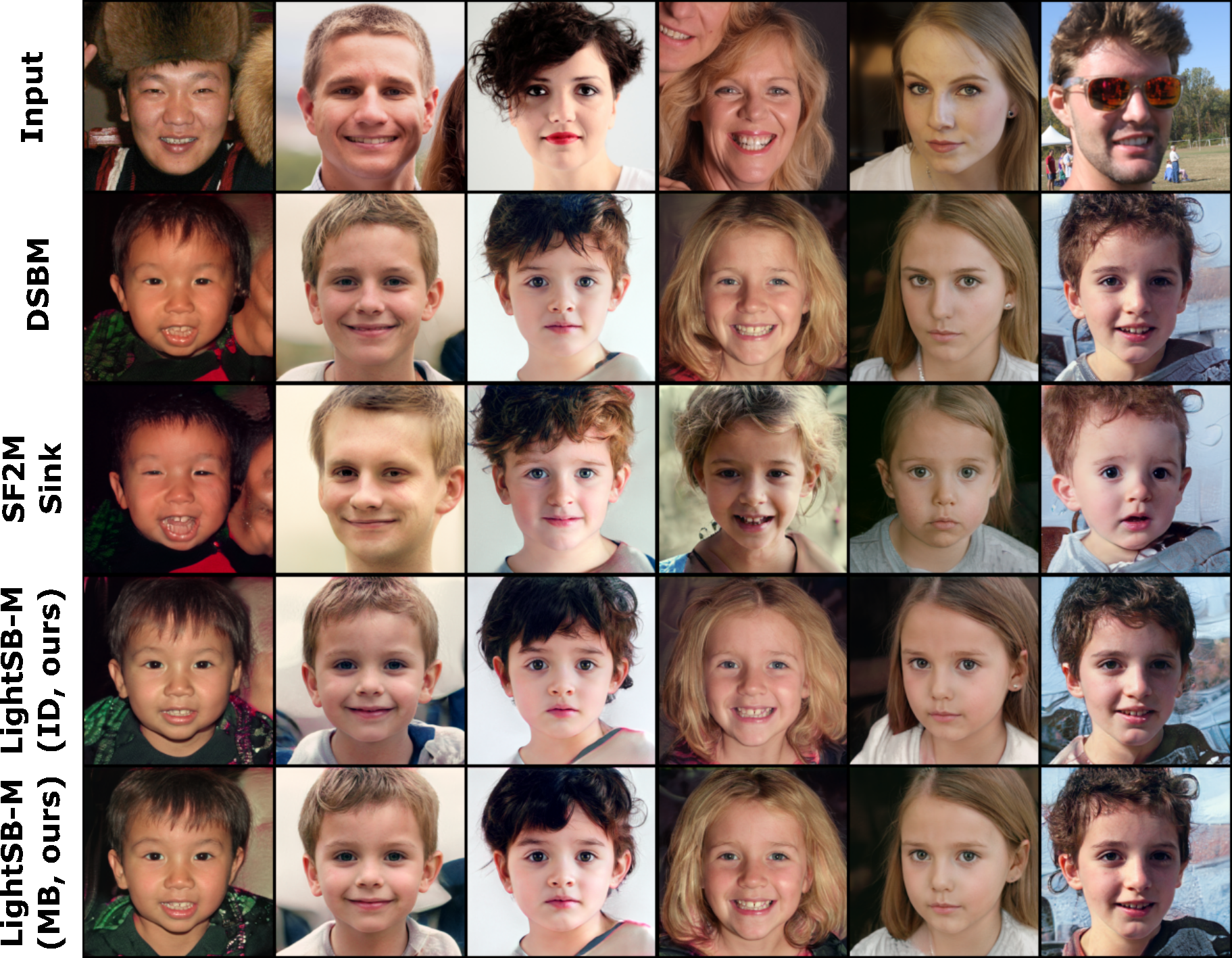}
\caption{\centering Adult $\rightarrow$ Child}
\vspace{-1mm}
\end{subfigure}
\vspace{-1mm}\hfill\begin{subfigure}[b]{0.49\linewidth}
\centering
\includegraphics[width=0.995\linewidth]{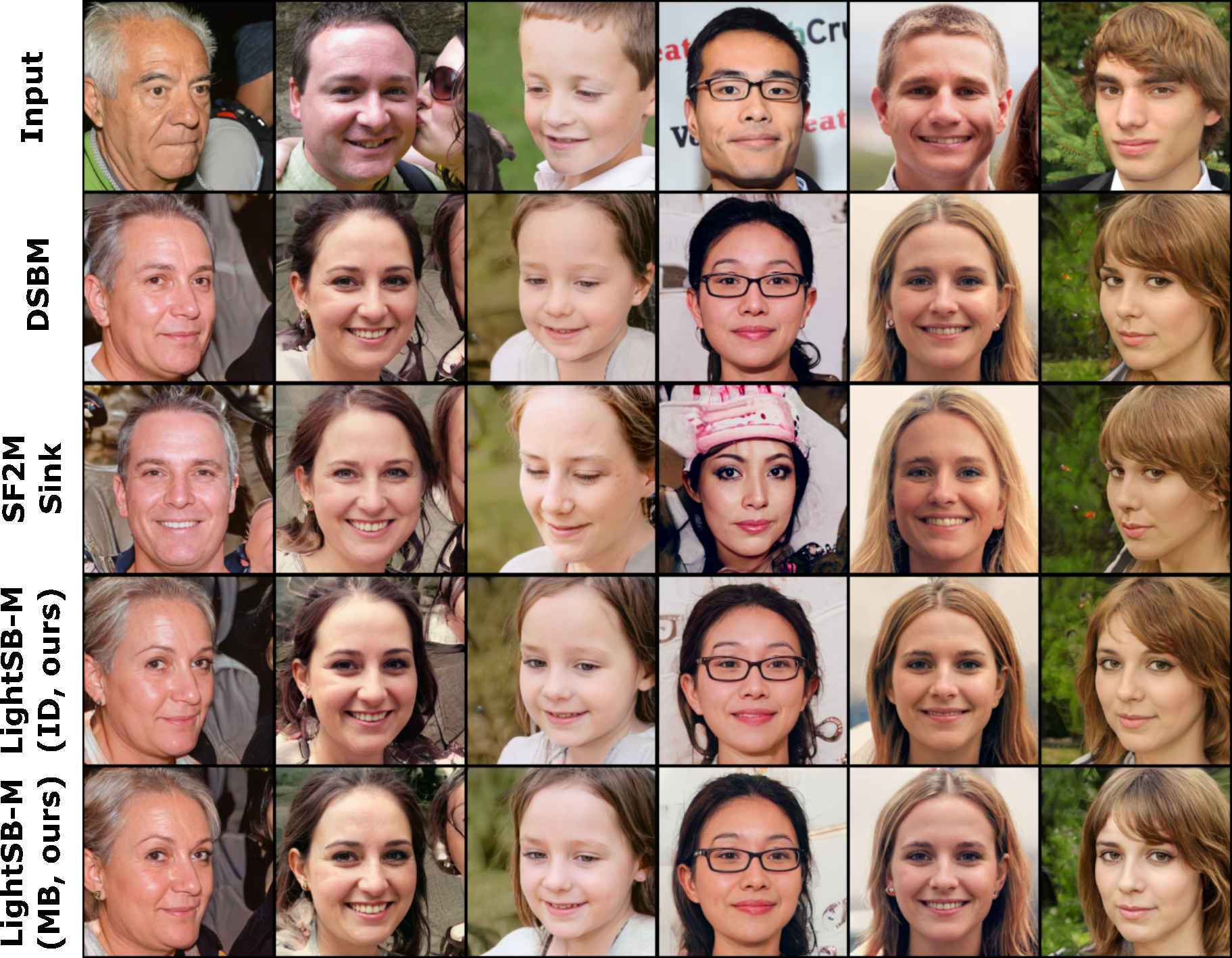}
\caption{\centering Man $\rightarrow$ Woman}
\vspace{-1mm}
\end{subfigure}
\vspace{-0mm} \caption{\centering Unpaired translation between subsets of FFHQ dataset (1024x1024) performed by various SB solvers (\wasyparagraph\ref{sec:exp-image}) in the latent space of ALAE \citep{pidhorskyi2020adversarial}.}\label{fig:alae}
\vspace{-4mm}
\end{figure*}

We see that our LightSB-M's solution with independent (ID) and minibatch discrete OT (MB) plans for $T_{\pi}$ provides the same metrics since it learns the same solution, as follows from the developed theory. It also shows performance on the same level as other neural network-based matching methods such as DSBM and $\text{SF}^2$M-Sink, but converges faster even without using GPU similar to the LightSB solver.

In Appendix~\ref{app:single-cell}, we provide the technical details for this setup and \underline{additional results} for different values of $\epsilon$.

\vspace{-1mm}
\subsection{Comparison on Unpaired Image-to-image Transfer} \label{sec:exp-image}
\vspace{-1mm}

Another popular setup that involves learning a translation between two distributions without paired data is image-to-image translation \citep{zhu2017unpaired}. Methods based on SB show promising results in solving this problem thanks to the perfect theoretical agreement of this setup with the SB formulation \citep{shi2023diffusion}. Due to the used parameterization based on Gaussian mixture, learning the translation between low-dimensional image manifolds is difficult for ${\text{LightSB-M}}$. Fortunately, many approaches use autoencoders \citep{rombach2022high} for more efficient generation and translation. We follow the setup of \cite{korotin2024light} with the pre-trained ALAE autoencoder \cite{pidhorskyi2020adversarial} on $1024 \times 1024$ FFHQ dataset \citep{karras2019style}. 

% \vspace{-1mm}
We present the qualitative results of our solver with discrete minibatch OT plan (MB) and independent plan (ID) in Fig~\ref{fig:alae}. For comparison, we also provide results of DSBM and $\text{SF}^2$M-Sink. Our LightSB-M solver converges to nearly the same solution for both ID and MB plans and demonstrates good results. The samples provided by DSBM are close to the samples of LightSB-M, which is expected since both methods provide theoretical guarantees for solving the SB problem. Samples obtained by $\text{SF}^2$M-Sink slightly differ, probably due to the bias of the discrete EOT plans. We provide \underline{additional examples} of translation in Appendix~\ref{sec-details-alae}. The \underline{details of the baselines} are given in Appendix~\ref{app:baselines}.

\vspace{-1mm}
\section{Discussion}
\vspace{-1mm}
\textbf{Potential impact.} Our main contribution is methodological: we show that one may perform just a single (but \textit{optimal}) bridge matching step to learn SB. This finding helps us eliminate limitations of existing bridge matching-based approaches, such as heuristical minibatch OT approximations or error accumulation during training. We believe that this insight is a significant step towards developing novel efficient computational approaches for SB/EOT tasks. 

% \vspace{-1mm}
\textbf{Limitations.} Given an adjusted Schrödinger potential $v$, it may be not easy to compute the drift $g_{v}$ \eqref{sb-drift} of $S_{v}$ needed to perform the optimal SB matching. We employ the Gaussian mixture parameterization for $v$ for which this drift $g_{v}$ is analytically known \eqref{sb-drift-gaussian}. This allows to easily implement our optimal SB matching in practice and obtain a fast bridge matching based solver. Still such a parameterization sometimes may be not sufficient, e.g.,  for large-scale generative modeling tasks. We point to developing ways to use more general parameterization of $v$ to our optimal SB matching, e.g., \underline{neural-network}-based, as a promising research avenue. We show possible steps in this direction in Appendix~\ref{app:neural-network-parametrization}.

One other limitation of our LightSB-M solver is that it is applicable to a limited set of priors. In this paper, we only consider the Wiener prior, which is one of the most popular priors used for SB. However, our method can be applied to other priors by changing the variables. These include Arithmetic Brownian Motion and Geometric Brownian Motion, also known as the Black-Scholes model, which is widely used in mathematical finance. Developing light solvers for Scr\"{o}dinger Bridges with more general priors is a promising direction for the future research.

\vspace{-1mm}
\section*{Acknowledgements}
\vspace{-1mm}
The work was supported by the Analytical center under the RF Government (subsidy agreement 000000D730321P5Q0002, Grant No. 70-2021-00145 02.11.2021).

\section*{Impact Statement}
This paper presents work whose goal is to advance the field of Machine Learning. There are many potential societal consequences of our work, none which we feel must be specifically highlighted here.

\bibliography{references}
\bibliographystyle{icml2024}

%%%%%%%%%%%%%%%%%%%%%%%%%%%%%%%%%%%%%%%%%%%%%%%%%%%%%%%%%%%%%%%%%%%%%%%%%%%%%%%
%%%%%%%%%%%%%%%%%%%%%%%%%%%%%%%%%%%%%%%%%%%%%%%%%%%%%%%%%%%%%%%%%%%%%%%%%%%%%%%
% APPENDIX
%%%%%%%%%%%%%%%%%%%%%%%%%%%%%%%%%%%%%%%%%%%%%%%%%%%%%%%%%%%%%%%%%%%%%%%%%%%%%%%
%%%%%%%%%%%%%%%%%%%%%%%%%%%%%%%%%%%%%%%%%%%%%%%%%%%%%%%%%%%%%%%%%%%%%%%%%%%%%%%
\newpage
\appendix
\onecolumn

\section{Proofs}\label{app:proofs}
\begin{proof}[Proof of Theorem \ref{thm:optimal-projection}]
Let $p_0,p_1$ denote the marginals of $\pi$. Let $\pi^*$ be the EOT plan between $p_0,p_1$. Let $p^S_0, p^S_1$ denote the distribution of $S$ at $t=0$ and $t=1$, respectively. We use the fact that each element $S$ of $\mathcal{S}$ is a reciprocal process with some EOT plan $\pi^S \in \Pi(p_0^S, p_1^S)$, i.e. $S=\int W^{\epsilon}_{|x_0, x_1}d\pi^S(x_0,x_1)$, recall \wasyparagraph{\ref{sec-background-sb}}. In turn, $\pi^S$ can be represented through the input density $p^S_0$ and the potential $v^S$ as in \eqref{plan-parametric-full}, i.e.:

\begin{equation}\label{eq:plan-S-form}
\pi^S(x_0, x_1) = p^S_0(x_0)\frac{\exp\big(\sfrac{\langle x_0,x_1\rangle}{\epsilon}\big)v^S(x_1)}{c_{v^{S}}(x_0)} 
\end{equation}

We derive:
\begin{eqnarray}
    \KL{T_{\pi}}{S}  = \KL{\pi}{\pi^{S}} + \int_{\mathbb{R}^D \times \mathbb{R}^D} \KL{T_{\pi|x_0,x_1}}{S_{|x_0,x_1}} \pi(x_0, x_1)dx_0dx_1  =
    \label{eq:op-proff-KL-disintegration}
    \\
    \KL{\pi}{\pi^{S}} + \int_{\mathbb{R}^D \times \mathbb{R}^D} \underbrace{\KL{W^{\epsilon}_{|x_0,x_1}}{W^{\epsilon}_{|x_0,x_1}}}_{=0} \pi(x_0, x_1)dx_0dx_1  = 
    \label{eq:op-proff-KL-inner-zero}
    \\
    \int_{\mathbb{R}^D \times \mathbb{R}^D } \frac{\log \pi(x_0,x_1)}{\log \pi^{S}(x_0,x_1)} \pi(x_0,x_1)dx_0dx_1  =
    \nonumber
    \\
    \int_{\mathbb{R}^D \times \mathbb{R}^D } \log \pi(x_0,x_1) \pi(x_0,x_1)dx_0dx_1 - \int_{\mathbb{R}^D \times \mathbb{R}^D } \log \pi^{S}(x_0,x_1) \pi(x_0,x_1)dx_0dx_1   = 
    \nonumber
    \\
    -H(\pi) - \int_{\mathbb{R}^D \times \mathbb{R}^D } \log \pi^{S}(x_0,x_1) \pi(x_0,x_1)dx_0dx_1  =
    \label{eq:op-proof-KL-pre-plan}
    \\
    -H(\pi)  - \int_{\mathbb{R}^D \times \mathbb{R}^D } \log \Big(p_0^S(x_0)\frac{\exp\Big(\sfrac{\langle x_0,x_1\rangle}{\epsilon}\Big)v^S(x_1)}{c_{v^{S}}(x_0)}\Big) \pi(x_0,x_1)dx_0dx_1  = 
    \label{eq:op-proof-KL-plan}
    \\
    -H(\pi) - \int_{\mathbb{R}^D \times \mathbb{R}^D } \Big( \log p_0^S(x_0) + \langle x_0,x_1\rangle + \log v^{S}(x_1) - \log c_{v^{S}}(x_0) \Big) \pi(x_0,x_1)dx_0dx_1 =
    \nonumber
    \\
    -H(\pi) - \int_{\mathbb{R}^D \times \mathbb{R}^D } \langle x_0,x_1\rangle \pi(x_0,x_1)dx_0dx_1  
    \nonumber
    \\
    - \int_{\mathbb{R}^D \times \mathbb{R}^D } \Big( \log p_0^S(x_0) - \log c_{v^{S}}(x_0) \Big) \pi(x_0,x_1)dx_0dx_1 - \int_{\mathbb{R}^D \times \mathbb{R}^D } \log v^{S}(x_1) \pi(x_0,x_1)dx_0dx_1 =
    \nonumber
    \\
    -H(\pi) -\int_{\mathbb{R}^D \times \mathbb{R}^D } \langle x_0,x_1\rangle \pi(x_0,x_1)dx_0dx_1 - \int_{\mathbb{R}^D} \log v^{S}(x_1) \underbrace{\big( \int_{\mathbb{R}^D} \pi(x_0|x_1)dx_0\big)}_{=1 = \int_{\mathbb{R}^D} \pi^*(x_0|x_1)dx_0} \underbrace{\pi(x_1)}_{=\pi^*(x_1)} dx_1
    \nonumber
    \\
    - \int_{\mathbb{R}^D} \big\{ ( \log p_0^S(x_0) - \log c_{v^{S}}(x_0)) \underbrace{\int_{\mathbb{R}^D} \pi(x_1|x_0)dx_1}_{=1 = \int_{\mathbb{R}^D} \pi^*(x_1|x_0)dx_1}\big\}\underbrace{\pi(x_0)}_{=\pi^*(x_0)}dx_0   = 
    \nonumber
    \\
    -H(\pi) -\int_{\mathbb{R}^D \times \mathbb{R}^D } \langle x_0,x_1\rangle \pi(x_0,x_1)dx_0dx_1 - \int_{\mathbb{R}^D} \log v^{S}(x_1) \big( \int_{\mathbb{R}^D} \pi^*(x_0|x_1)dx_0\big) \pi^*(x_1) dx_1
    \nonumber
    \\
    - \int_{\mathbb{R}^D} \big\{ ( \log p_0^S(x_0) - \log c_{v^{S}}(x_0)) \int_{\mathbb{R}^D} \pi^*(x_1|x_0)dx_1 \big\}\pi^*(x_0)dx_0 = 
    \nonumber
    \\
    -H(\pi) -\int_{\mathbb{R}^D \times \mathbb{R}^D } \langle x_0,x_1\rangle \pi(x_0,x_1)dx_0dx_1 - \int_{\mathbb{R}^D \times \mathbb{R}^D } \log v^{S}(x_1) \pi^*(x_0,x_1)dx_0dx_1
    \nonumber
    \\
    - \int_{\mathbb{R}^D \times \mathbb{R}^D } \Big( \log p_0^S(x_0) - \log c_{v^{S}}(x_0) \Big) \pi^*(x_0,x_1)dx_0dx_1 =
    \nonumber
    \\
    -H(\pi) -\int_{\mathbb{R}^D \times \mathbb{R}^D } \langle x_0,x_1\rangle \pi(x_0,x_1)dx_0dx_1 
    \nonumber
    \\
    + \underbrace{\int_{\mathbb{R}^D \times \mathbb{R}^D } \langle x_0,x_1\rangle \pi^*(x_0,x_1)dx_0dx_1 -\int_{\mathbb{R}^D \times \mathbb{R}^D } \langle x_0,x_1\rangle \pi^*(x_0,x_1)dx_0dx_1}_{=0}
    \nonumber
    \\
    - \int_{\mathbb{R}^D \times \mathbb{R}^D } \Big( \log p_0^S(x_0) + \log v^{S}(x_1) - \log c_{v^{S}}(x_0) \Big) \pi^*(x_0,x_1)dx_0dx_1 = 
    \nonumber
    \\
    -H(\pi) -\int_{\mathbb{R}^D \times \mathbb{R}^D } \langle x_0,x_1\rangle \pi(x_0,x_1)dx_0dx_1 + \int_{\mathbb{R}^D \times \mathbb{R}^D } \langle x_0,x_1\rangle \pi^*(x_0,x_1)dx_0dx_1
    \nonumber
    \\
    - \int_{\mathbb{R}^D \times \mathbb{R}^D } \Big( \log p_0^S(x_0) + \langle x_0,x_1\rangle + \log v^{S}(x_1) - \log c_{v^{S}}(x_0) \Big) \pi^*(x_0,x_1)dx_0dx_1 = 
    \nonumber
    \\
    -H(\pi) -\int_{\mathbb{R}^D \times \mathbb{R}^D } \langle x_0,x_1\rangle \pi(x_0,x_1)dx_0dx_1 + \int_{\mathbb{R}^D \times \mathbb{R}^D } \langle x_0,x_1\rangle \pi^*(x_0,x_1)dx_0dx_1
    \nonumber
    \\
    - \int_{\mathbb{R}^D \times \mathbb{R}^D } \log \Big(p_0^S(x_0)\underbrace{\frac{\exp\Big(\sfrac{\langle x_0,x_1\rangle}{\epsilon}\Big)v^S(x_1)}{c_{v^{S}}(x_0)}}_{\pi^{S}(x_1|x_0)}\Big) \pi^*(x_0,x_1)dx_0dx_1 =
    \nonumber
    \\
    -H(\pi) -\int_{\mathbb{R}^D \times \mathbb{R}^D } \langle x_0,x_1\rangle (\pi(x_0, x_1) - \pi^*(x_0,x_1))dx_0dx_1 - \int_{\mathbb{R}^D \times \mathbb{R}^D } \log \pi^S(x_0, x_1) \pi^*(x_0,x_1)dx_0dx_1 
    \nonumber
    \\
    + \underbrace{\int_{\mathbb{R}^D \times \mathbb{R}^D } \log \pi^*(x_0, x_1) \pi^*(x_0, x_1)dx_0dx_1 - \int_{\mathbb{R}^D \times \mathbb{R}^D } \log \pi^*(x_0, x_1) \pi^*(x_0, x_1)dx_0dx_1}_{=0} =
    \nonumber
    \\
    \underbrace{-H(\pi) -\int_{\mathbb{R}^D \times \mathbb{R}^D } \langle x_0,x_1\rangle (\pi(x_0, x_1) - \pi^*(x_0,x_1))dx_0dx_1 - \int_{\mathbb{R}^D \times \mathbb{R}^D } \log \pi^*(x_0, x_1) \pi^*(x_0,x_1)dx_0dx_1}_{\defeq \widehat{C}(\pi)}
    \nonumber
    \\
    + \underbrace{\int_{\mathbb{R}^D \times \mathbb{R}^D } \log \frac{\pi^*(x_0, x_1)}{\pi^S(x_0, x_1)} \pi^*(x_0,x_1)dx_0dx_1}_{= \KL{\pi^*}{\pi^S}} = 
    \nonumber
    \\
    \widehat{C}(\pi) + \KL{\pi^*}{\pi^S}.
    \nonumber
\end{eqnarray}
In \eqref{eq:op-proff-KL-disintegration} we use disintegration theorem for KL divergence to distinguish process plan $\pi^S$ and "inner part" \citep[Appendix C, D]{vargas2021solving}. In transition from \eqref{eq:op-proff-KL-disintegration} to \eqref{eq:op-proff-KL-inner-zero} we notice, that $T_{\pi|x_0, x_1}=W^{\epsilon}_{|x_0, x_1}$ and $S_{|x_0, x_1} = W^{\epsilon}_{|x_0, x_1}$, since $T_{\pi}$ is a reciprocal process as well as Schrödinger Bridge $S$. In transition from \eqref{eq:op-proof-KL-pre-plan} to \eqref{eq:op-proof-KL-plan} we use the fact, that $\pi^S$ is given by \eqref{eq:plan-S-form}. Since
\begin{eqnarray}
    \KL{T_{\pi}}{S} = \widehat{C}(\pi) + \KL{\pi^*}{\pi^S},
    \nonumber
\end{eqnarray}
the minimum of $\KL{T_{\pi}}{S}$ is achieved for $S$ such that $\pi^S = \pi^*$, i.e., when $S$  is the SB between $p_0$ and $p_1$.

\end{proof}

\begin{proof}[Proof of Theorem \ref{thm:tractable-objective-of-optimal-projection}]
We start by using a Pythagorean theorem for Markovian projection \citep[Lemma 6]{shi2023diffusion}
\begin{eqnarray}
    \KL{T_{\pi}}{S_v} = \KL{T_{\pi}}{\text{proj}_{\mathcal{M}}(T_{\pi})} + \KL{\text{proj}_{\mathcal{M}}(T_{\pi})}{S_v},
\end{eqnarray}
where the drift $g_{\mathcal{M}}$ of the Markoian projection $\text{proj}_{\mathcal{M}}(T_{\pi})$ is given by \eqref{eq:markovian-proj-drift}:
\begin{eqnarray}
    g_{\mathcal{M}}(x_t, t) = \int_{\mathbb{R}^D} \frac{x_1 - x_t}{1-t} dp_{T_{\pi}}(x_1|x_t).
    \label{eq:second-proof-drift}
\end{eqnarray}
We use the expression of KL between Markovian processes starting from the same distribution $p_0$ through their drifts \citep{pavon1991free} and note that Markovian projection preserve the time marginals $p_{T_{\pi}}(x_t)$:
\begin{eqnarray}
    \KL{\text{proj}_{\mathcal{M}}(T_{\pi})}{S_v} = \frac{1}{2\epsilon} \int_{0}^1 \int_{\mathbb{R}^D}|| g_v(x_t, t) - g_{\mathcal{M}}(x_t, t)||^2 dp_{T_{\pi}}(x_t) dt
\end{eqnarray}
Then we substitute $g_{\mathcal{M}}$ by \eqref{eq:second-proof-drift}:
\begin{eqnarray}
    \KL{\text{proj}_{\mathcal{M}}(T_{\pi})}{S_v} = \frac{1}{2\epsilon} \int_{0}^1 \int_{\mathbb{R}^D} || g_v(x_t, t) - g_{\mathcal{M}}(x_t, t)||^2 dp_{T_{\pi}}(x_t)dt =
    \nonumber
    \\
    \frac{1}{2\epsilon} \int_{0}^1 \int_{\mathbb{R}^D} || g_v(x_t, t) - \int_{\mathbb{R}^D} \frac{x_1 - x_t}{1-t} dp_{T_{\pi}}(x_1|x_t) ||^2 dp_{T_{\pi}}(x_t)dt =
    \nonumber
    \\
    % \frac{1}{2\epsilon} \int_{0}^1 \int_{\mathbb{R}^D} \Big\{  || g_v(x_t, t)||^2 - 2 \langle g_v(x_t, t), \int_{\mathbb{R}^D} \frac{x_1 - x_t}{1-t} dp_{T_{\pi}}(x_1|x_t) \rangle + ||\int_{\mathbb{R}^D} \frac{x_1 - x_t}{1-t} dp_{T_{\pi}}(x_1|x_t)||^2   \Big\} dp_{T_{\pi}}(x_t)dt 
    % \nonumber
    % \\
    \frac{1}{2\epsilon} \int_{0}^1 \int_{\mathbb{R}^D} \Big\{  || g_v(x_t, t)||^2 - 2 \langle g_v(x_t, t), \int_{\mathbb{R}^D} \frac{x_1 - x_t}{1-t} dp_{T_{\pi}}(x_1|x_t) \rangle \Big\}dp_{T_{\pi}}(x_t)dt + 
    \nonumber
    \\
    \underbrace{\frac{1}{2\epsilon} \int_{0}^1 \int_{\mathbb{R}^D} ||\int_{\mathbb{R}^D} \frac{x_1 - x_t}{1-t} dp_{T_{\pi}}(x_1|x_t)||^2 dp_{T_{\pi}}(x_t)dt}_{\defeq C'(\pi)} =
    \nonumber
    \\
    \frac{1}{2\epsilon} \int_{0}^1 \int_{\mathbb{R}^D} \Big\{  || g_v(x_t, t)||^2 - 2 \langle g_v(x_t, t), \int_{\mathbb{R}^D} \frac{x_1 - x_t}{1-t} dp_{T_{\pi}}(x_1|x_t) \rangle \Big\}dp_{T_{\pi}}(x_t)dt + C'(\pi) =
    \nonumber
    \\
    \frac{1}{2\epsilon} \int_{0}^1 \int_{\mathbb{R}^D} \int_{\mathbb{R}^D} \Big\{|| g_v(x_t, t)||^2 - 2 \langle g_v(x_t, t),  \frac{x_1 - x_t}{1-t}\rangle \Big\}\underbrace{dp_{T_{\pi}}(x_1|x_t)dp_{T_{\pi}}(x_t)}_{dp_{T_{\pi}}(x_t, x_1)}dt + C'(\pi) =
    \nonumber
    \\
    \frac{1}{2\epsilon} \int_{0}^1 \int_{\mathbb{R}^D} \int_{\mathbb{R}^D} \Big\{|| g_v(x_t, t)||^2 - 2 \langle g_v(x_t, t),  \frac{x_1 - x_t}{1-t}\rangle \Big\}dp_{T_{\pi}}(x_t, x_1)dt + C'(\pi) =
    \nonumber
    \\
    \frac{1}{2\epsilon} \int_{0}^1 \int_{\mathbb{R}^D} \int_{\mathbb{R}^D} \Big\{|| g_v(x_t, t) - \frac{x_1 - x_t}{1-t}||^2 \Big\}dp_{T_{\pi}}(x_t, x_1)dt - 
    \nonumber
    \\
    \underbrace{\frac{1}{2\epsilon} \int_{0}^1 \int_{\mathbb{R}^D} \int_{\mathbb{R}^D} ||\frac{x_1 - x_t}{1-t}||^2 dp_{T_{\pi}}(x_t, x_1)dt + C'(\pi)}_{\defeq C''(\pi)} =
    \nonumber
    \\
    \frac{1}{2\epsilon} \int_{0}^1 \int_{\mathbb{R}^D} \int_{\mathbb{R}^D} \Big\{|| g_v(x_t, t) - \frac{x_1 - x_t}{1-t}||^2 \Big\}dp_{T_{\pi}}(x_t, x_1)dt + C''(\pi)
    \nonumber
\end{eqnarray}
Thus,
\begin{eqnarray}
    \KL{T_{\pi}}{S_v} = \underbrace{\KL{T_{\pi}}{\text{proj}_{\mathcal{M}}(T_{\pi})} + C'(T_{\pi})}_{\defeq C(\pi)} + 
    \nonumber
    \\
    \frac{1}{2\epsilon} \int_{0}^1 \int_{\mathbb{R}^D} \int_{\mathbb{R}^D} \Big\{|| g_v(x_t, t) - \frac{x_1 - x_t}{1-t}||^2 \Big\}dp_{T_{\pi}}(x_t, x_1)dt = 
     \nonumber
    \\
    C(\pi) + \frac{1}{2\epsilon} \int_{0}^1 \int_{\mathbb{R}^D} \int_{\mathbb{R}^D} \Big\{|| g_v(x_t, t) - \frac{x_1 - x_t}{1-t}||^2 \Big\}dp_{T_{\pi}}(x_t, x_1)dt.
\end{eqnarray}
\end{proof}

\begin{proof}[Proof of Theorem \ref{thm:eqviv-lightsb}]
From Theorem~\ref{thm:tractable-objective-of-optimal-projection} it follows that:
\begin{eqnarray}
    \KL{T_{\pi}}{S_v} = C(\pi) + \frac{1}{2\epsilon} \int_{0}^1 \int_{\mathbb{R}^D} \int_{\mathbb{R}^D} \Big\{|| g_v(x_t, t) - \frac{x_1 - x_t}{1-t}||^2 \Big\}dp_{T_{\pi}}(x_t, x_1)dt.
\end{eqnarray}
In turn, from the proof of Theorem~\eqref{thm:optimal-projection} it holds that:
\begin{eqnarray}
    \KL{T_{\pi}}{S_v} = \widehat{C}(\pi) + \KL{\pi^*}{\pi^{S_{v}}}.
\end{eqnarray}
From the \citep[Propositon 3.1]{korotin2024light} it follows that $\KL{\pi^*}{\pi^{S}_{v}} = \mathcal{L}_0(v) - \mathcal{L}^*$, where $\mathcal{L}^*$ is a constant depending on distirbutions $p_0, p_1$ and value $\epsilon$.
Hence, we combine these two expressions and get
\begin{eqnarray}
    \frac{1}{2\epsilon} \int_{0}^1 \int_{\mathbb{R}^D} \int_{\mathbb{R}^D} \Big\{|| g_v(x_t, t) - \frac{x_1 - x_t}{1-t}||^2 \Big\}dp_{T_{\pi}}(x_t, x_1)dt = 
    \widetilde{C}(\pi) + \mathcal{L}_0(v),
    \nonumber
\end{eqnarray}
where $\widetilde{C}(\pi) \defeq \widehat{C}(\pi) - \mathcal{L}^* - C(\pi)$.

\end{proof}

% \subsection{Quantitative Evaluation on the SB Benchmark}

\section{Experiments details and extra results}\label{app:solver-details}

% \subsection{General Implementation Details}

We build our LightSB-M implementation upon LightSB official implementation  \url{https://github.com/ngushchin/LightSB}. All the parametrization, optimization and initialization details are the same as \cite{korotin2024light} if not stated otherwise. In the Mini-batch (MB) setting, discrete OT algorithm \texttt{ot.emd} is taken from POT library \cite{flamary2021pot}. The batch size is always 128.

\subsection{Qualitative 2D setup hyperparameters}\label{app:hyperparameters-for-2D}

We use $K=250$ potentials and Adam optimizer with $lr=10^{-3}$ in all the cases to train LightSB-M.

\subsection{Evaluation on Biological Single-cell Data.}\label{app:single-cell}
We follow the same setup as \citep{korotin2024light} and use their code and data from \url{https://github.com/ngushchin/LightSB}. All models are trained with $\epsilon=0.1$ if not stated otherwise. For completeness, we provide additional results of our solver trained with the independent plan with different values of the parameter $\epsilon$, see Table \ref{app:table-sc-comparison-eps}. 
% As in the main part of the paper (where we follow related works and use $\epsilon=0.1$), we use the Energy distance as the evaluation metric. The results for new considered $\epsilon$ are rather close to the results which we initially reported for $\epsilon=0.1$.

\color{black}
\begin{table*}[h]
\vspace{-0mm}
% \color{blue}
\centering
\small
\begin{tabular}{ |c|c|c|c| }
\hline
\backslashbox{$\epsilon$}{\textbf{DIM}} & \textbf{50} & \textbf{100} & \textbf{1000} \\ 
\hline
 $0.3$ & $2.37 \pm 0.11$ & $2.169 \pm 0.11$  & $1.310 \pm 0.06$  \\
\hline
 $0.1$ & $2.347 \pm 0.11$ & $2.174 \pm 0.08$ & $1.35 \pm 0.05$ \\
\hline
 $0.03$ & $2.349 \pm 0.09$ & $2.32 \pm 0.09$ & $1.279 \pm 0.05$ \\
\hline
 $0.01$ & $2.404 \pm 0.12$ & $2.28 \pm 0.07$ & $1.309 \pm 0.04$ \\
\hline
\end{tabular}
\vspace{-2mm}
\captionsetup{justification=centering}
 \caption{Energy distance (averaged for two setups and 5 random seeds) on the MSCI dataset (\wasyparagraph\ref{sec:exp-single-cell}) along with $95\%$-confidence interval ($\pm$ intervals) for LightSB-M (ID).}
 \label{app:table-sc-comparison-eps}
% \vspace{-5mm}
\end{table*}

\subsection{Evaluation on the Schrodinger Bridge Benchmark}\label{app:benchmark}

Here we first provide an additional evaluation of solvers using target matching and dynamic metrics. Then we study the speed of inference in our LightSB-M solver using the Brownian bridge vs. using the Euler–Maruyama simulation.

\textbf{Target metric evaluation.} We additionally study how well each solver map initial distribution $p_0$ into $p_1$ by measuring the metric $\text{B}\mathbb{W}_{2}^{2}\text{-UVP}$ also proposed by the authors of the benchmark \citep[\wasyparagraph{4}]{gushchin2023building}. We present the results in Table~\ref{table-bwuvp-benchmark}. We observe that our method performs better than other bridge-matching approaches.

\begin{table*}[h]
\centering
\scriptsize
\setlength{\tabcolsep}{3pt}
\begin{tabular}{rrcccccccccccc}
\toprule
& & \multicolumn{4}{c}{$\epsilon=0.1$} & \multicolumn{4}{c}{$\epsilon=1$} & \multicolumn{4}{c}{$\epsilon=10$}\\
\cmidrule(lr){3-6} \cmidrule(lr){7-10} \cmidrule(l){11-14}
& Solver Type & {$D\!=\!2$} & {$D\!=\!16$} & {$D\!=\!64$} & {$D\!=\!128$} & {$D\!=\!2$} & {$D\!=\!16$} & {$D\!=\!64$} & {$D\!=\!128$} & {$D\!=\!2$} & {$D\!=\!16$} & {$D\!=\!64$} & {$D\!=\!128$} \\
\midrule
Best solver on benchmark$^\dagger$ & Varies & $0.016$ & $0.05$ & $0.25$ & $0.22$ & $0.005$ & $0.09$  & $0.56$ & $0.12$ & $0.01$ & $0.02$ & $0.15$ & $0.23$ \\ 
LightSB$^\dagger$ & KL minimization & $0.005$ & $0.017$ & $0.037$ & $0.069$ & $0.004$ & $0.01$ & $0.03$ & $0.07$ & $0.03$ & $0.04$ & $0.17$ & $0.30$ \\
\hline
DSBM & \multirow{5}{*}{Bridge matching} & $0.03$ & $0.18$ & $0.7$ & $2.26$ & $0.04$ & $0.09$ & $1.9$ & $7.3$ & $0.26$ & $102$ & $3563$ & $15000$ \\   
$\text{SF}^2$M-Sink &  &$0.04$ & $0.18$ & $0.39$ & $1.1$ & $0.07$ & $0.3$ & $4.5$ & $17.7$ & $0.17$ & $4.7$ & $316$ & $812$ \\   
LightSB-M (ID, \textbf{ours}) &  & $0.02$ & $\mathbf{0.03}$ & $\mathbf{0.2}$ & $\mathbf{0.46}$ & $0.005$ & $0.04$ & $\mathbf{0.11}$ & $0.27$ & $0.07$ & $\mathbf{0.03}$ & $0.11$ & $\mathbf{0.21}$ \\  
LightSB-M (MB, \textbf{ours}) &  & $\mathbf{0.005}$ & $0.07$ & $0.27$ & $0.63$ & $\mathbf{0.002}$ & $0.04$ & $0.12$ & $0.36$ & $0.04$ & $0.07$ & $0.11$ & $0.23$ \\   
LightSB-M (GT, \textbf{ours}) &  & $0.02$ & $\mathbf{0.03}$ & $0.21$ & $0.55$ & $0.011$ & $\mathbf{0.03}$ & $\mathbf{0.11}$ & $\mathbf{0.26}$ & $\mathbf{0.016}$ & $0.04$ & $\mathbf{0.09}$ & $\mathbf{0.21}$ \\
\bottomrule
\end{tabular}
\vspace{-2mm}
\captionsetup{justification=centering,font=small}
\caption{Comparisons of $\text{B}\mathbb{W}_{2}^{2}\text{-UVP}\downarrow$ (\%) between the ground truth target distribution $p_1$ and learned target distribution $\pi_{\theta}(x_1)$. The best metric over \textit{bridge matching} solvers is \textbf{bolded}. Results marked with $\dagger$ are taken from \cite{korotin2024light}.}
\label{table-bwuvp-benchmark}
\vspace{-2mm}
\end{table*}

% In our paper, we focus on the cBW-UVP static part of the metric since we use parametrization such that for every parameter $\theta$, the corresponding $g_{\theta}(x_t, t)$ is the drift function of some Schrödinger Bridge starting in $p_0$. Hence, if the learned plan $\pi^S(x_0, x_1)$ is learned right, then $g_{\theta}(x_t, t)$ is the drift function of the Schrödinger Bridge starting in $p_0^S = p_0$ and ending in approximately $p_1^S \approx p_1$. Hence, in our case, the dynamical metric is highly correlated with the static cBW-UVP metric thanks to the parametrization we used and is a sort of redundant.

\textbf{Dynamic metrics evaluation}. Following the authors of the benchmark paper \citep[Appendix F]{gushchin2023building}, we provide additional metrics for the learned dynamic of the Schrödinger Bridge. The authors of the benchmark measure forward $\text{KL}(T^*||S)$ and reversed $\text{KL}(S||T^*)$ divergences between the ground-truth process $T^*$ and the learned process $S$. To do so, they define two auxiliary values:
$$
\mathcal{L}^{2}_{\text{fwd}}[t] = \mathbb{E}_{x_t \sim T^*}\|g^{*}(x_t,t)-g_S(x_t,t)\|^{2}, \quad
\mathcal{L}^{2}_{\text{rev}}[t] = \mathbb{E}_{x_t \sim S}\|g^{*}(x_t,t)-g_S(x_t,t)\|^{2},
$$
and use the fact, that $\text{KL}(T^*||S) = \frac{1}{2\epsilon}\int_{0}^{1}\mathcal{L}^{2}_{\text{fwd}}[t]dt$ and $\text{KL}(S||T^*) = \frac{1}{2\epsilon}\int_{0}^{1}\mathcal{L}^{2}_{\text{rev}}[t]dt$.
% define two values $\mathcal{L}^{2}_{\text{fwd}}[t]$ and $\mathcal{L}^{2}_{\text{rev}}[t]$ which are the MSE between ground-truth Schrödinger Bridge drift $g^*(X_t, t)$ and the learned model drift $g_S(X_t, t)$ but averaged over process $T^*$ and $S$ correspondingly:
% $$
% \mathcal{L}^{2}_{\text{fwd}}[t] = \mathbb{E}_{X_t \sim T^*}\|g^{*}(X_t,t)-g_S(X_t,t)\|^{2},
% $$
% $$
% \mathcal{L}^{2}_{\text{rev}}[t] = \mathbb{E}_{X_t \sim S}\|g^{*}(X_t,t)-g_S(X_t,t)\|^{2}.
% $$

The values of $\mathcal{L}^{2}_{\text{fwd}}[t]$ and $\mathcal{L}^{2}_{\text{rev}}[t]$ for the range of $t \in [0, 1]$ are plotted in Figure \ref{fig:kl_dyn_sb_bench}. We observe that LightSB-M has a lower error $\mathcal{L}^{2}_{\text{fwd}}[t]$ and $\mathcal{L}^{2}_{\text{rev}}[t]$ as well as $\text{KL}(T^*||S)$ and $\text{KL}(S||T^*)$ in approximating the ground-truth optimal drift $g^*(X_t, t)$ than other algorithm including DSBM \citep{shi2023diffusion} and SF2M \citep{tong2023simulation}. Values of $\text{KL}(T^*||S)$ and $\text{KL}(S||T^*)$ are given in the Table \ref{tab:kl_dyn_values_sb_bench}. Since $\mathcal{L}^{2}_{\text{fwd}}[t]$ and $\mathcal{L}^{2}_{\text{rev}}[t]$ are lower for most times for our algorithm, $\text{KL}(T^*||S)$ and $\text{KL}(S||T^*)$ are lower for our \textbf{LightSB-M} algorithm.

\begin{figure}[!h]
    \centering
    \includegraphics[width=0.955\linewidth]{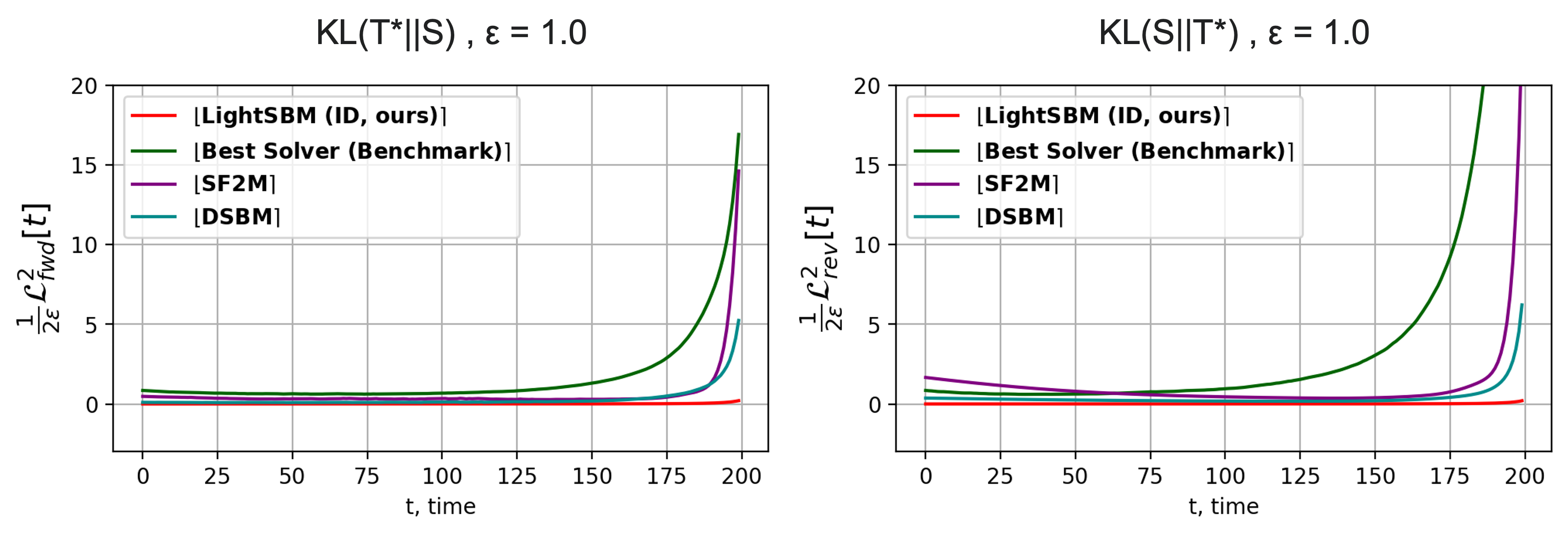}
    \captionsetup{justification=centering}
    \caption{Dynamic KL evaluation. $\mathcal{L}^{2}_{\text{fwd}}[t]$ and $\mathcal{L}^{2}_{\text{bwd}}[t]$ values w.r.t. time for different algorithms. Results denoted as "Best solver (benchmark)" are taken from the benchmark paper \citep{gushchin2023building}}
    \label{fig:kl_dyn_sb_bench}
\end{figure}

\begin{table}[!h]
    \centering
    \begin{tabular}{|c|c|c|c|c|}
        \hline
        Solver & LightSBM (ID, ours) & Best solver (benchmark) & SF2M & DSBM \\
        \hline
        $\text{KL}(T^*||S)$ & 0.0093 & 1.64 & 0.6422 & 0.2950 \\
        \hline
        $\text{KL}(S\||T^*)$ & 0.0099 & 49.65 & 1.0765 & 0.39 \\
        \hline
    \end{tabular}
    \captionsetup{justification=centering}
    \caption{Dynamic KL values for different algorithms. Results denoted as "Best solver (benchmark)" are taken from the benchmark paper \citep{gushchin2023building}.}
    \label{tab:kl_dyn_values_sb_bench}
\end{table}

\textbf{Study of the efficiency of the sampling.} Here we measure the performance of sampling (time and $\text{cB}\mathbb{W}_{2}^{2}\text{-UVP}\downarrow$) directly from the learned plan $\pi_{\theta}(x_1|x_0)$ versus sampling by the Euler-Maruyama algorithm \citep[\wasyparagraph{9.2}]{kloeden1992stochastic} and using the drift function $g_{\theta}(x_t, t)$. We conduct our experiments on the benchmark setup with $\epsilon=0.1$ and $D=16$. We present our results in the Table \ref{tab:plan_vs_sde_sampling_time} and Table \ref{tab:plan_vs_sde_sampling_cbwuvp}:

\begin{table}[!h]
    \centering
    \begin{tabular}{|c|c|}
        \hline
        \textbf{Inference type} & \textbf{Time} \\
        \hline
        Euler-Maruyama, 3 steps & 0.046 $\pm$ 0.053 sec \\
        \hline
        Euler-Maruyama, 10 steps & 0.19 $\pm$ 0.14 sec \\
        \hline
        Euler-Maruyama, 30 steps & 0.365 $\pm$ 0.08 sec \\
        \hline
        Euler-Maruyama, 100 steps & 1.268 $\pm$ 0.3 sec \\
        \hline
        Euler-Maruyama, 300 steps & 3.931 $\pm$ 0.34 sec \\
        \hline
        Euler-Maruyama, 1000 steps & 12.61 $\pm$ 1.32 sec \\
        \hline
        Sampling from the plan $\pi_{\theta}$ & 0.00058 $\pm$ 0.0001 sec \\
        \hline
    \end{tabular}
    \captionsetup{justification=centering}
    \caption{Time measurements for LightSB-M sampling using the SDE approach (Euler-Maruyama) and direct sampling from the plan $\pi_{\theta}$ on SB Benchmark \citep{gushchin2023building} with $\epsilon=0.1$ and $D=16$. The number of steps for Euler-Maruyama is the number of SDE solver discretization steps. Results are averaged over 5 runs with std provided after $\pm$.}
    \label{tab:plan_vs_sde_sampling_time}
\end{table}

\begin{table}[!h]
    \centering
    \begin{tabular}{|c|c|}
        \hline
        Inference type & $\text{cB}\mathbb{W}_{2}^{2}\text{-UVP}\downarrow$ \\
        \hline
        Euler-Maruyama, 10 steps & 1.53  \\
        \hline
        Euler-Maruyama, 50 steps & 0.22  \\
        \hline
        Euler-Maruyama, 100 steps & 0.126 \\
        \hline
        Euler-Maruyama, 200 steps & 0.102  \\
        \hline
        Euler-Maruyama, 500 steps & 0.09  \\
        \hline
        Sampling from the plan $\pi_{\theta}$ & 0.09  \\
        \hline
    \end{tabular}
    \captionsetup{justification=centering}
    \caption{$\text{cB}\mathbb{W}_{2}^{2}\text{-UVP}\downarrow$ measurements for LightSB-M sampling using SDE approach (Euler-Maruyama) and sampling from the plan $\pi_{\theta}$ on SB Benchmark \citep{gushchin2023building} with $\epsilon=0.1$ and $D=16$. Number of steps for Euler-Maruyama means number of SDE solver discretization steps.}
    \label{tab:plan_vs_sde_sampling_cbwuvp}
\end{table}

As we can see from the obtained results, the Euler-Maruyama approach requires up to 500 steps to accurately solve the Schrödinger Bridge SDE. Thanks to the special form of the SDE provided by the used parametrization, we can directly sample from $\pi_{\theta}(x_1|x_0)$. This is orders of magnitude faster than the full simulation of the trajectories.

\subsection{Evaluation on unpaired image-to-image translation.}\label{sec-details-alae}

We follow the same setup as \citep{korotin2024light} and use their code and data from \url{https://github.com/ngushchin/LightSB}. All models are trained with $\epsilon=0.1$ if not stated otherwise.

According to \cite{korotin2024light} we first split the FFHQ data into train (first 60k) and test (last 10k) images. Then we create subsets of \textit{males}, \textit{females}, \textit{children} and \textit{adults} in both train and test subsets. For training we first use the ALAE encoder to extract $512$ dimensional latent vectors for each image and then train our solver on the extracted latent vectors. At the inference stage, we first extract the latent vector from the image, translate it by LightSB-M, and then decode the mapped vector to produce the mapped image. In Figure~\ref{fig:additiona-exmp}, we provide extra examples for our LigthSB-M and other baselines.

\textbf{Test FID values}. The FID values for our \textit{man} $\rightarrow$ \textit{woman} FFHQ image translation setup are provided in Table \ref{tab:fid_alae}. We measure FID values between decoded translated latents and encoded-decoded true images from the FFHQ dataset. For all considered solvers, we use the same value of the coefficient $\epsilon = 0.1$, which produces moderate diversity in the generated images. The FID values are similar for all methods, which align with the good quality of images given in Figure \ref{fig:alae}.

\begin{table}[h]
    \centering
    \begin{tabular}{|c|c|c|c|c|}
    \hline
        Solver & LightSBM (ID, \textbf{ours}) & LightSBM (MB, \textbf{ours}) & DSBM & SF$^2$M-Sink\\
        \hline
        FID & 0.852 & 0.859 & 0.859 & 0.8613\\
        \hline
    \end{tabular}
    \caption{FID values on unpaired \textit{man} $\rightarrow$ \textit{woman} translation for different solvers applied in the latent space of ALAE \citep{pidhorskyi2020adversarial} for 1024x1024 FFHQ images. \citep{karras2019style}}
    \label{tab:fid_alae}
\end{table}

\textbf{Different values of $\epsilon$.} We provide extra \textit{male} $\rightarrow$ \textit{female} results for a wide range of values $\epsilon \in \{0.01, 0.1, 1, 10\}$) in the Figure \ref{fig:appx_lightsbm_diversity} below. We observe that our solver shows the expected behavior by providing more diversity for larger $\epsilon$. 

% We follow the work of \citep{korotin2024light} and use the official ALAE code and model from 
% \begin{center}
%     \url{https://github.com/podgorskiy/ALAE}
% \end{center} and neural network extracted attributes for the FFHQ dataset from 
% \begin{center}
% \url{https://github.com/DCGM/ffhq-features-dataset}
% \end{center}
% According to \cite{korotin2024light} we first split the FFHQ data into train (first 60k) and test (last 10k) images. Then we create subsets of \textit{males}, \textit{females}, \textit{children} and \textit{adults} in both train and test subsets. For training we first use the ALAE encoder to extract $512$ dimensional latent vectors for each image and then train our solver on the extracted latent vectors. At the inference stage, we first extract the latent vector from the image, translate it by LightSB-M, and then decode the mapped vector to produce the mapped image.

\subsection{Baselines}\label{app:baselines}

\textbf{DSBM} \citep{shi2023diffusion}. Implementation is taken from official repo
\begin{center}\url{https://github.com/yuyang-shi/dsbm-pytorch}
\end{center}
For forward and backward drift approximations, instead of those used in the official repository, we use MLP neural networks with positional encoding as they give better results. Number of inner gradient steps for Markovian Fitting Iteration is $10000$, number of  Markovian Fitting Iterations is $10$. Adam optimizer \cite{kingma2014adam} with $lr=10^{-4}$ is used for optimization. 
\textbf{$\text{SF}^2$M-Sink} \citep{tong2023simulation}. Implementation is taken from official repo 
\begin{center}
\url{https://github.com/atong01/conditional-flow-matching}
\end{center}
For drift and score function approximation, we use MLP neural networks with positional encoding instead of those used in the official repository, as they give better results. Number of gradient updates $50000$ for SB benchmark and Single-cell Data experiments and $20000$ for unpaired image-to-image translation. Adam optimizer \cite{kingma2014adam} with $lr=10^{-4}$ is used for optimization.
\textbf{LightSB}'s results are taken from the paper \cite{korotin2024light}.

\section{Neural Network parametrization}\label{app:neural-network-parametrization}

Our LightSB-M is based on the Gaussian mixture parametrization. However, it is not the only way to implement optimal projection in practice. Here, we additionally propose a method to use neural network parametrization. We call this modification of our algorithm Hard Schrödinger Bridge Matching, or \textbf{HardSB-M}.

To begin with, we discuss another type of Schrödinger potential for better clarity. In the main text, we utilize the more convenient adjusted Schrödinger potential $v$ since it simplifies the usage of Gaussian Mixture approximation. However, in the literature, the more popular way is the usage of the Schrödinger potential $\varphi$ \citep[Eq. 4.11]{chen2021stochastic}, which has a one-to-one correspondence with the adjusted Schrödinger potential $v$:
\begin{equation}
    \varphi(x, t=1) = \varphi(x) = v(x)\exp(\frac{||x||^2}{2\epsilon}).
\end{equation}
Furthermore, the drift $g(x, t)$ of the Schrödinger Bridge with potential $\varphi(x)$ is given by:
\begin{eqnarray}
    g(x_t, t) = \epsilon \nabla_{x_t}\! \log \!\int_{\mathbb{R}^{D}}\!\!\!\! \mathcal{N}(x'|x_t, (1-t)\epsilon I_{D})\exp\!\big(\frac{\|x'\|^{2}}{2\epsilon}\big) v(x') dx' = 
    \nonumber
    \\
    \epsilon \nabla_{x_t}\! \log \!\int_{\mathbb{R}^{D}}\!\!\!\! \mathcal{N}(x'|x_t, (1-t)\epsilon I_{D}) \varphi(x') dx'.
    \label{eq:not-conv-with-gaussian-mixture}
\end{eqnarray}

Below we use this non adjusted Schrödinger potential and denote $\varphi(x, t=1)$ as $\varphi(x)$ to make derivations more concise. We parametrize this potential by a neural network $\varphi_{\theta}(x)$ and use \eqref{eq:not-conv-with-gaussian-mixture} to derive the drift $g_{\theta}$ given by $\varphi_{\theta}(x)$.

% of \underline{"non adjusted"} Schr\"{o}dinger potential $\varphi_{\theta}$. 

\subsection{Drift Estimation}

In the case of neural parametrization of $\varphi_{\theta}$ (or $v_{\theta}$), computation of drift $g_{\theta}(x_t, t)$ \eqref{eq:not-conv-with-gaussian-mixture} becomes a non-trivial task since it is no longer a convolution of a Gaussian mixture with a Gaussian distribution. We propose two ways to tackle this issue.

\textbf{\underline{Variant 1.} Monte Carlo (MC) estimator}. \label{app:hardsbm_mc_estimator}
First, we recap SB drift expression \eqref{eq:not-conv-with-gaussian-mixture} which for parametrized Schrödinger potential $\varphi_{\theta}$ states that the drift $g_{\theta}(x_t, t)$ is given by:
$$
    g_{\theta}(x_t, t) = \epsilon \nabla_{x_t}\! \log \!\int_{\mathbb{R}^{D}}\!\!\!\! \mathcal{N}(x'|x_t, (1-t)\epsilon I_{D}) \varphi_{\theta}(x') dx'
$$
By using the reparametrization trick (introducing $x' \defeq z\sqrt{(1-t)\epsilon} + x_t$), we get
\begin{eqnarray}
g_{\theta}(x_t, t) = \epsilon \nabla_{x_t}\! \log \!\int_{\mathbb{R}^{D}}\!\!\!\! \mathcal{N}(z|0, I_{D}) \varphi_{\theta}(z\sqrt{(1-t)\epsilon} + x_t)dz =
\nonumber
\\
\epsilon \frac{\!\int_{\mathbb{R}^{D}} \nabla_{x_t}  \big( \varphi_{\theta}(z\sqrt{(1-t)\epsilon} + x_t)\big) \mathcal{N}(z|0, I_{D})dz}{\int_{\mathbb{R}^{D}} \big( \varphi_{\theta}(z\sqrt{(1-t)\epsilon} + x_t)\big) \mathcal{N}(z|0, I_{D})dz} 
= \epsilon \frac{\mathbb{E}_{z \sim \mathcal{N}(z|0, I_{D})}\big[ \nabla_{x_t}  \big( \varphi_{\theta}(z\sqrt{(1-t)\epsilon} + x_t)\big) \big]}{\mathbb{E}_{z \sim \mathcal{N}(z|0, I_{D})} \big[ \big( \varphi_{\theta}(z\sqrt{(1-t)\epsilon} + x_t)\big) \big]}.
\nonumber
\end{eqnarray}

Then we can estimate $g_{\theta}(x_t, t)$ just by drawing samples $\{z\}_{n=1}^N$ and $\{z\}_{m=1}^M$ from $\mathcal{N}(z|0, I_{D})$ and using:

$$
g_{\theta}(x_t, t) \approx \frac{\frac{1}{N}\sum_{n=1}^N  \nabla_{x_t}  \big( \varphi_{\theta}(z_n\sqrt{(1-t)\epsilon} + x_t)\big)}{\frac{1}{M}\sum_{m=1}^M  \big( \varphi_{\theta}(z_n\sqrt{(1-t)\epsilon} + x_t)\big)}
$$

Calculation of the gradient of loss $\KL{T_{\pi}}{S_{\varphi_{\theta}}}$ given by \eqref{eq:mse-optimal-projection} w.r.t.\ the parameters is straightforward using auto-differentiation software.

\textbf{\underline{Variant 2.} Monte Carlo Markov Chain (MCMC) estimator}.  \label{app:hardsbm_mcmc_estimator}
The MC estimator proposed above is biased. We also suggest a non-biased estimator based on sampling from the unnormalized density below.

\begin{theorem}[\textbf{HardSB-M drift expression}] \label{th:hardsbm_drift}
    The drift $g(x, t)$ for the Schrodinger potential $\varphi(x)$ is given by:
    \begin{equation}
        g(x_t, t) = \frac{1}{1-t}\big(\mathbb{E}_{x' \sim p^{\varphi}(x'|x_t)}[x'] - x_t\big),
        \label{eq:mcmc-drift-estimator}
    \end{equation}
    where $p^{\varphi}(x'|x_t) \propto \exp{(-\frac{\Vert x' - x_t \Vert^2}{2\epsilon (1-t)})}\varphi(x')$.
\end{theorem}

% \begin{theorem}[\textbf{HardSB-M MCMC drift estimation}] \label{th:hardsbm_drift}
%     For a given Schrödinger Bridge potential $\varphi_{\theta}(x)$ then corresponding SBP with Wiener Prior, $W^\epsilon$, solution $S_{\varphi}$ yields the following SDE:
    
%     $$T_{\theta}:dX_t = g_{\theta}(X_t, t)dt + \sqrt{\epsilon}dW_t, X_0 \sim p_0$$
    
%     $$g_{\theta}(x_t, t) = \frac{1}{1-t}(\mathbb{E}_{x' \sim p(x'|x_t, \varphi_{\theta})}[x'] - x_t)$$

%     where $p(x'|x_t, \varphi_{\theta}) \propto \exp{(-\frac{\Vert x' - x_t \Vert^2}{2\epsilon (1-t)})}\varphi_{\theta}(x')$.
    
% \end{theorem}

To estimate drift by Theorem \ref{th:hardsbm_drift}, one needs to sample from unnormalized density $p^{\varphi}(x'|x_t)$. 
% Following best practices of energy-based models \cite{song2021train}, we use general MCMC methods, 
To do this, one may use the standard Unadjusted Langevin Algorithm (ULA), also known as just Langevin Dynamics \cite{LangevinRbertsTweedie96}.

\subsection{Loss Gradient Estimation}

To optimize the objective \eqref{eq:mse-optimal-projection}, one needs to compute its gradient w.r.t.\ the parameters $\theta$ which also involves $\nabla_{\theta} g_{\theta}(x_t, t)$. With the MC estimator for $g_{\theta}(x_t, t)$, the gradient $\nabla_{\theta} g_{\theta}(x_t, t)$ is trivially computed using automatic differentiation. However, with the MCMC estimator proposed in Theorem~\ref{th:hardsbm_drift} the way to compute $\nabla_{\theta} g_{\theta}(x_t, t)$ is not trivial. We propose an unbiased gradient of loss estimator via sampling from unnormalized density. For this, we need the following theorem.
% In this derivation, we slightly abuse our notation. Specifically, we note that the parameters $\theta$ in $g_{\theta}$ are the parameters of the corresponding potential $\varphi_{\theta}$, i.e. $\nabla_{\theta}g_{\theta}$ is a vector of size of number of parameters used for parametrizing $\varphi_{\theta}$ and not the Jacobian.

\begin{theorem}[\textbf{HardSB-M loss gradient expression}] \label{th:hardsbm_loss_grad}
The gradient of \eqref{eq:mse-optimal-projection} for the Schrodinger potential $\varphi_{\theta}(x)$ is given by:
    $$
        \frac{1}{\epsilon} \int_0^1 \int_{\mathbb{R}^D \times \mathbb{R}^D} \Big\{(\nabla_{\theta} g_{\theta}(x_t, t))^{\top}(g_{\theta}(x_t, t) - \frac{x_1 - x_t}{1 - t})  \Big\} dp_{T_{\pi}}(x_t, x_1)dt,
    $$
where
    $$\nabla_{\theta} g_{\theta}(x_t, t) = \frac{1}{1-t}\nabla_{\theta} \mathbb{E}_{p^{\varphi_{\theta}}(x'|x_t)}[x'].$$
In turn, $\nabla_{\theta}\mathbb{E}_{p^{\varphi_{\theta}}(x_t)}[x']$ can be computed via
$$\nabla_{\theta} \mathbb{E}_{p^{\varphi_{\theta}}(x'|x_t)}[x'] = \mathbb{E}_{p^{\varphi_{\theta}}(x'|x_t)} \Big[ x' \{ \nabla_{\theta} \log \varphi_{\theta}(x') - \mathbb{E}_{x'' \sim p^{\varphi_{\theta}}(x''| x_t)}[\nabla_{\theta} \log \varphi_{\theta}(x'')] \} \Big].$$
\end{theorem}

% \begin{theorem}[\textbf{HardSB-M MCMC Optimal Projection loss gradient}] \label{th:hardsbm_loss_grad}
%      Consider Schr\"{o}dinger Potential $ \varphi_{\theta}(x) $ and corresponding Schrödinger Bridge $S_\varphi$. Denote $p(x'|x_t, \varphi_{\theta}) \propto \exp{(-\frac{\Vert x' - x \Vert^2}{2\epsilon (1-t)})}\varphi_{\theta}(x')$ as $p'_{\theta}(x_t)$. Then the gradient of "Optimal Projection" optimization objective has the form:

%       $$\nabla_{\theta}L_{\theta}(T_{\pi}) = \frac{1}{\epsilon} \int_0^1 \int_{\mathbb{R}^D \times \mathbb{R}^D} \Big\{ \Vert g_{\theta}(x_t, t) - \frac{x_1 - x_t}{1 - t} \Vert *\nabla_{\theta} g_{\theta}(x_t, t) \Big\} dp_{T_{\pi}}(x_t, x_1)dt$$

% where

%       $$ \nabla_{\theta} g_{\theta}(x_t, t) = \frac{1}{1-t}(\nabla_{\theta} \mathbb{E}_{p'_{\theta}(x_t)}[x'] - x_t) $$

% and finally

% $$\nabla_{\theta} \mathbb{E}_{p'_{\theta}(x_t)}[x'] = \mathbb{E}_{p'_{\theta}(x_t)} \Big[ x' \{ \nabla_{\theta} \log \varphi_{\theta}(x') - \mathbb{E}_{x'' \sim p'_{\theta}(x_t)}[\nabla_{\theta} \log \varphi_{\theta}(x'')] \} \Big]$$
     
% \end{theorem}

To use this theorem in practice, we first estimate $g_{\theta}(x_t, t)$ by MCMC using \eqref{eq:mcmc-drift-estimator} from Theorem~\ref{th:hardsbm_drift}. Then, we estimate the gradient of the objective by using the Theorem~\ref{th:hardsbm_loss_grad}. At both stages, samples from $p^{\varphi_{\theta}}(x'|x_t)$ can be drawn, e.g., using the Unadjusted Langevin Algorithm \citep[ULA]{LangevinRbertsTweedie96}. 

% Both of these approaches allow the implementation of HardSB-M in practice.

\subsection{Inference after model training}

There are several inference approaches, e.g., \textbf{Energy-Based} and \textbf{SDE Based}. 

\textbf{Energy based inference.} We can sample directly from the EOT plan \eqref{eot-plan-characterization} using generic MCMC samplers similar to EgNOT \citep{mokrov2024energyguided}. After sampling the end point $x_1$ given the start point $x_0$, the trajectories can be infered using self similarity property \citep[\wasyparagraph{3.2}]{korotin2024light} of the Brownian Bridge $W^{\epsilon}_{|x_0,x_1}$.

\textbf{SDE based inference.} Given the way to estimate drift $g_{\theta}$ of the Schrödinger Bridge e.g. by MC (Appendix~\ref{app:hardsbm_mc_estimator}) or MCMC (Appendix~\ref{app:hardsbm_mcmc_estimator}) approaches, one can use any SDE solver to simulate trajectories. For example, one can use the simplest and most popular Euler-Maruyama scheme \citep[\wasyparagraph{9.2}]{kloeden1992stochastic}.

One can combine these approaches by sampling an MCMC proposal used for Energy Based inference via SDE simulation.

\subsection{Toy 2D experimental illustration}

We use the same setups as in \wasyparagraph\ref{sec:exp-2D} with $\epsilon \in \{0.03, 0.1, 1\}$ and provide results for MC estimation in Figure~\ref{fig:hardsbm_swiss_roll_mc} and for MCMC estimation in Figure~\ref{fig:hardsbm_swiss_roll_mcmc}. The Schrödinger potential $ \varphi_{\theta}(x) : \mathbb{R}^D \rightarrow \mathbb{R}_+$ is parametrized using $\exp(\text{NN}_{\theta})$, where $\text{NN}_{\theta}$ is a MLP. We test both MC and MCMC approaches.

\textbf{Hyperparameters.} For both MC and MCMC estimators, we use MLP with two hidden layers of widths $[256, 256]$ with \texttt{torch.nn.SiLU} activations as $\text{NN}_{\theta}$. During training, Adam optimizer \citep{kingma2014adam} with $lr=10^{-4}$ is used, batch size is $128$, model is trained for $10^5$ loss gradient updates. 

\textbf{MC estimator.} During training and inference, we use $1000$ MC samples. Inference is held with SDE simulation using 1000 Euler-Maruyama discretization steps for $\epsilon=1$ and $100$ Euler-Maruyama discretization steps for other $\epsilon$. Due to the necessity to compute $\exp(\text{NN}_{\theta})$ which may have very high values, we use double precision \texttt{torch.DoubleTensor} for all MC-related calculations.

% \begin{figure}[t]
%     \centering
%     \includegraphics[width=0.9\linewidth]{pics/hardsbm/hardsbm_swiss_roll.png}
%     \caption{The process $S_{\theta}$ learned by HardSB-M with MCMC drift estimator (\textbf{ours}) \textit{Gaussian} $\!\rightarrow\!$ \textit{Swiss roll} example.}
%     \label{fig:hardsbm_swiss_roll_mcmc}
% \end{figure}

\begin{figure*}[!t]
% \vspace{-6mm}
\begin{subfigure}[b]{0.245\linewidth}
\centering
\includegraphics[width=0.995\linewidth]{pics/LSBM_swiss_roll_gt.png}
\caption{\centering ${x\sim p_0
}$, ${y \sim p_1}.$}
\vspace{-1mm}
\end{subfigure}
\vspace{-1mm}\hfill\begin{subfigure}[b]{0.245\linewidth}
\centering
\includegraphics[width=0.995\linewidth]{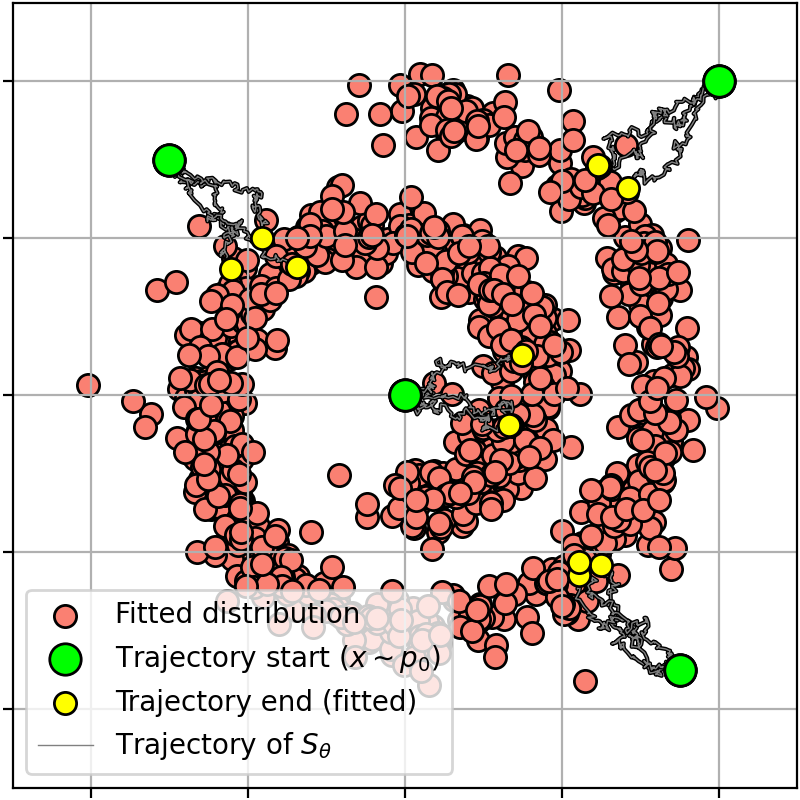}
\caption{\centering $\epsilon=0.03$.}
\vspace{-1mm}
\end{subfigure}
\hfill\begin{subfigure}[b]{0.245\linewidth}
\centering
\includegraphics[width=0.995\linewidth]{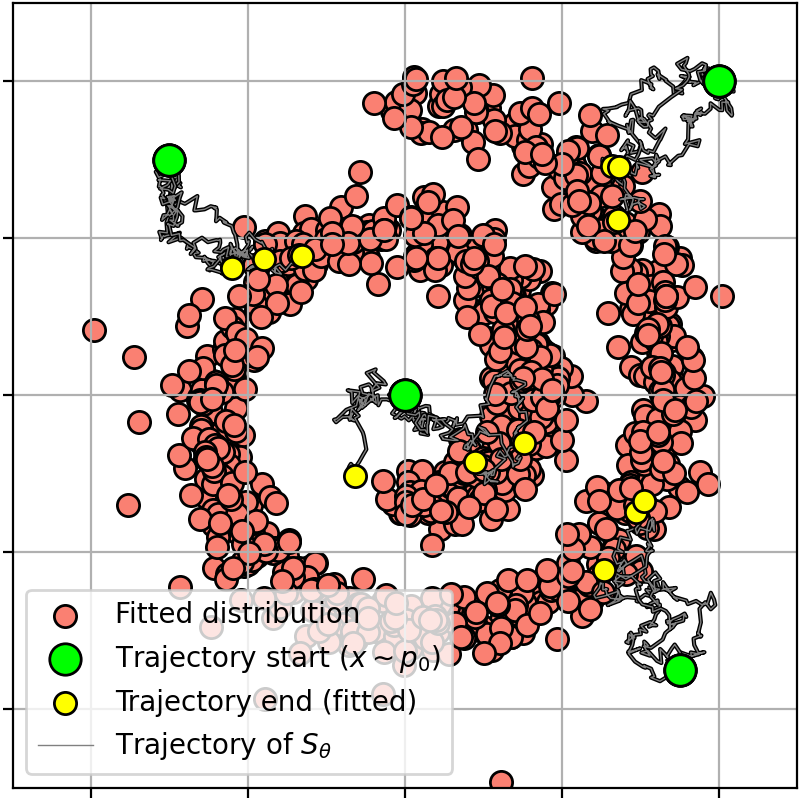}
\caption{\centering $\epsilon=0.1$.}
\vspace{-1mm}
\end{subfigure}
\hfill\begin{subfigure}[b]{0.245\linewidth}
\centering
\includegraphics[width=0.995\linewidth]{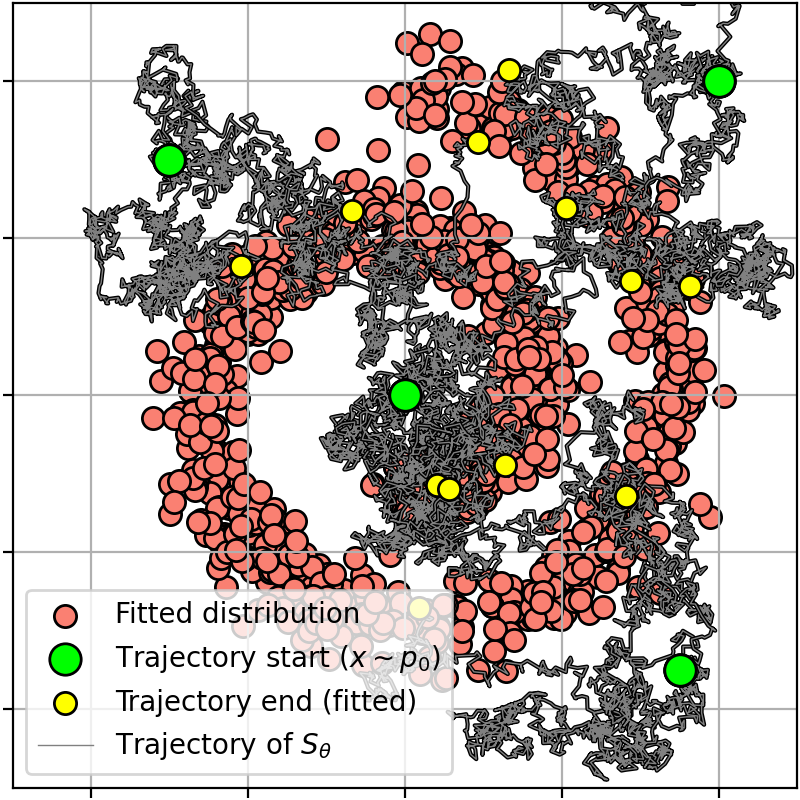}
\caption{\centering $\epsilon=1$.}
\vspace{-1mm}
\end{subfigure}
\vspace{-0mm} \caption{\centering The process $S_{\theta}$ learned by HardSB-M (\textbf{ours}) with MC drift estimator  \textit{Gaussian} $\!\rightarrow\!$ \textit{Swiss roll} example.}
 \label{fig:hardsbm_swiss_roll_mc}
\vspace{-3mm}
\end{figure*}

\begin{figure*}[!t]
% \vspace{-6mm}
\begin{subfigure}[b]{0.245\linewidth}
\centering
\includegraphics[width=0.995\linewidth]{pics/LSBM_swiss_roll_gt.png}
\caption{\centering ${x\sim p_0
}$, ${y \sim p_1}.$}
\vspace{-1mm}
\end{subfigure}
\vspace{-1mm}\hfill\begin{subfigure}[b]{0.245\linewidth}
\centering
\includegraphics[width=0.995\linewidth]{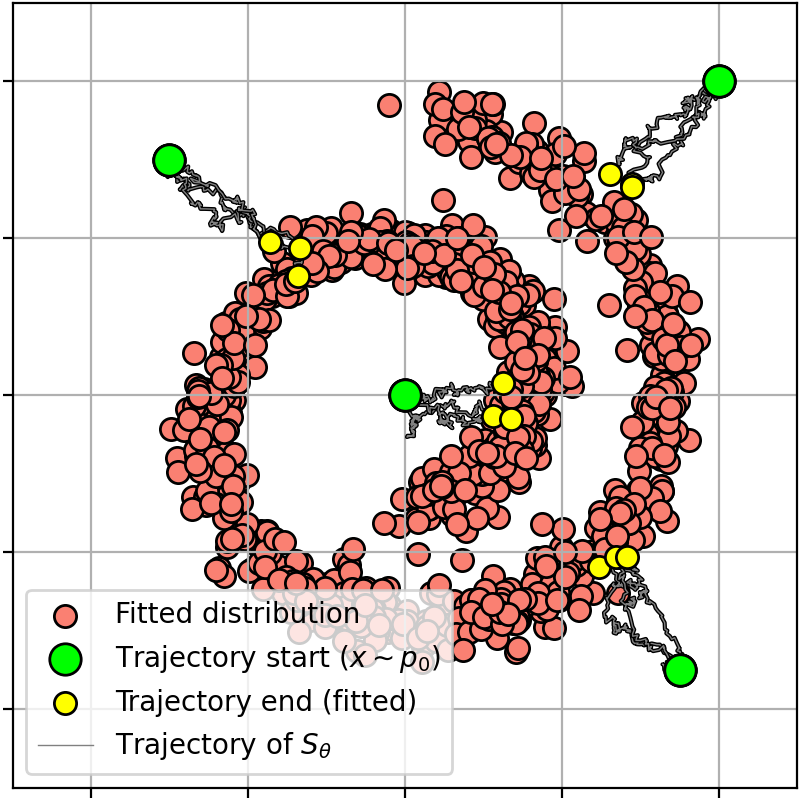}
\caption{\centering $\epsilon=0.03$.}
\vspace{-1mm}
\end{subfigure}
\hfill\begin{subfigure}[b]{0.245\linewidth}
\centering
\includegraphics[width=0.995\linewidth]{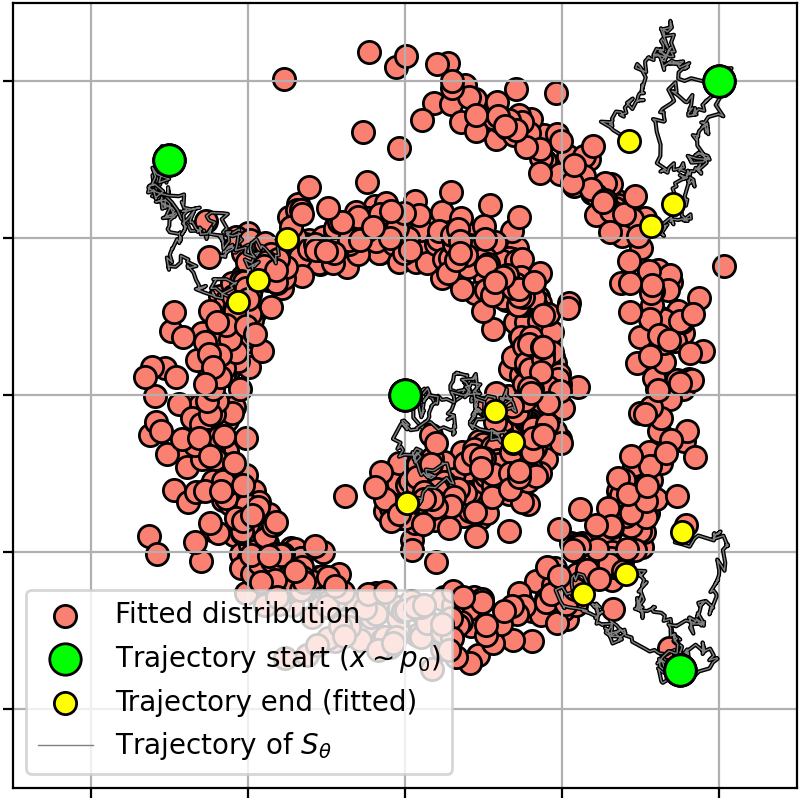}
\caption{\centering $\epsilon=0.1$.}
\vspace{-1mm}
\end{subfigure}
\hfill\begin{subfigure}[b]{0.245\linewidth}
\centering
\includegraphics[width=0.995\linewidth]{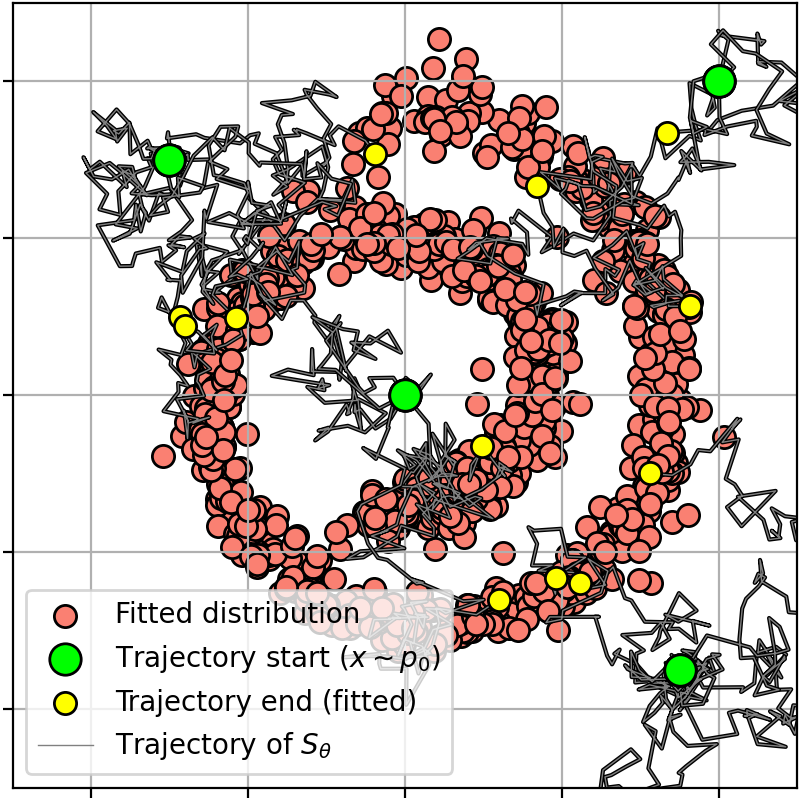}
\caption{\centering $\epsilon=1$.}
\vspace{-1mm}
\end{subfigure}
\vspace{-0mm} \caption{\centering The process $S_{\theta}$ learned by HardSB-M (\textbf{ours}) with MCMC drift estimator \textit{Gaussian} $\!\rightarrow\!$ \textit{Swiss roll} example.}
 \label{fig:hardsbm_swiss_roll_mcmc}
\vspace{-3mm}
\end{figure*}

\textbf{MCMC estimator.} During training and inference to estimate $g_{\theta}$ by \eqref{eq:mcmc-drift-estimator} we use $100$ samples drawn using Unadjusted Langevin Algorithm (ULA) with 50 steps and step size $\eta=0.0001$. The inference was performed in two steps: first, the SDE simulation was performed with $g_{\theta}$ estimation by ULA, and then the result was used as a proposal for energy-based sampling from the EOT plan, \eqref{eot-plan-characterization}. SDE simulaion was held using 100 Euler-Maruyama discretization steps (ULA settings are the same as for training) and Energy Based sampling from EOT plan using ULA with with 1000 steps and step size $\eta=10^{-4}$.

\subsection{Proofs}

\begin{proof}[Proof of Theorem \ref{th:hardsbm_drift}]
We denote by $Z_{x_t, (1-t)\epsilon} \defeq \int_{\mathbb{R}^{D}} \exp\big( -\frac{\Vert x' - x_t \Vert^2}{2(1-t)\epsilon } \big) dx'$ the normalization constant of the normal distribution $\mathcal{N}(x'|x_t, (1-t)\epsilon I_{D})$.
For a potential $\varphi_{\theta}(x)$, the corresponding drift $g(x_t, t)$ is given by \eqref{eq:not-conv-with-gaussian-mixture}:
\begin{eqnarray}
    g(x_t, t) = \epsilon \nabla_{x_t} \log \int_{\mathbb{R}^{D}} \mathcal{N}(x'|x_t, (1-t)\epsilon I_{D}) \varphi(x') dx'
    \nonumber
\end{eqnarray}
We will proceed with this equality to obtain an unbiased estimator. We proceed as follows:
\begin{eqnarray}
    g(x_t, t) = \epsilon \nabla_{x_t} \log \int_{\mathbb{R}^{D}} \mathcal{N}(x'|x_t, (1-t)) \varphi(x') dx' =
    \nonumber
    \\
    \epsilon \nabla_{x_t} \log \int_{\mathbb{R}^{D}} \frac{1}{Z_{x_t, (1-t)\epsilon}} \exp\big( -\frac{\Vert x' - x_t \Vert^2}{2(1-t)\epsilon } \big)  \varphi(x') dx'=
    \nonumber
    \\
    \epsilon \frac{\nabla_{x_t}  \int_{\mathbb{R}^{D}} \frac{1}{Z_{x_t, (1-t)\epsilon}} \exp\big( -\frac{\Vert x' - x_t \Vert^2}{2(1-t)\epsilon } \big)  \varphi(x') dx'}{\int_{\mathbb{R}^{D}} \frac{1}{Z_{x_t, (1-t)\epsilon}} \exp\big( -\frac{\Vert x'' - x_t \Vert^2}{2(1-t)\epsilon } \big)  \varphi(x'') dx''} =
    \nonumber
    \\
    \epsilon \frac{\nabla_{x_t}  \int_{\mathbb{R}^{D}} \exp\big( -\frac{\Vert x' - x_t \Vert^2}{2(1-t)\epsilon } \big)  \varphi(x') dx'}{\int_{\mathbb{R}^{D}}  \exp\big( -\frac{\Vert x'' - x_t \Vert^2}{2(1-t)\epsilon } \big)  \varphi(x'') dx''} =
    \nonumber
    \\
    \epsilon \frac{  \int_{\mathbb{R}^{D}} \nabla_{x_t} \Big\{ \exp\big( -\frac{\Vert x' - x_t \Vert^2}{2(1-t)\epsilon } \big)\Big\}  \varphi(x')  dx'}{\int_{\mathbb{R}^{D}}  \exp\big( -\frac{\Vert x'' - x_t \Vert^2}{2(1-t)\epsilon } \big)  \varphi(x'') dx''} = \Big[ \nabla_x f(x) = f(x) \nabla_x \log f(x) \Big] =
    \label{eq:supp-log-derivative-trick-3}
    \\
    \epsilon \int_{\mathbb{R}^{D}} \nabla_{x_t} \big\{ \frac{-\Vert x' - x_t \Vert^2}{2(1-t)\epsilon } \big\} \frac{  \exp\big( -\frac{\Vert x' - x_t \Vert^2}{2(1-t)\epsilon } \big)  \varphi(x') dx'}{\int_{\mathbb{R}^{D}}  \exp\big( -\frac{\Vert x'' - x_t \Vert^2}{2(1-t)\epsilon } \big)  \varphi(x'') dx''} = 
    \nonumber
    \\
    \Big[ p^{\varphi}(x'|x_t) \propto \exp{(-\frac{\Vert x' - x_t \Vert^2}{2\epsilon (1-t)})}\varphi(x')  \Big] =
    \nonumber
    \\
    \epsilon \int_{\mathbb{R}^{D}} \nabla_{x_t} \big\{ \frac{-\Vert x' - x_t \Vert^2}{2(1-t)\epsilon } \big\} p^{\varphi}(x'|x_t)dx' = 
    \nonumber
    \\
    \epsilon \mathbb{E}_{x' \sim p^{\varphi}(x'|x_t)}\big[ \frac{x' - x_t}{(1-t)\epsilon } \big] = \frac{1}{(1-t)} \big(\mathbb{E}_{x' \sim p^{\varphi}(x'|x_t)}[x'] - x_t \big).
\end{eqnarray}
In line \eqref{eq:supp-log-derivative-trick-3}, we use log-derivatve trick.
\end{proof}

\begin{proof}[Proof of Theorem \ref{th:hardsbm_loss_grad}]

% Proof of \textbf{HardSB-M MCMC Optimal Projection loss gradient Theorem \ref{th:hardsbm_loss_grad}}.
% The gradient of Optimal Projection optimization objective \eqref{eq:optimal_proj_loss} w.r.t. $\theta$ is expressed as follows:
We derive

\begin{eqnarray}
\nabla_\theta \frac{1}{\epsilon} \int_0^1 \int_{\mathbb{R}^D \times \mathbb{R}^D} \Vert g_{\theta}(x_t, t) - \frac{x_1 - x_t}{1 - t} \Vert^2 d p_{T_{\pi}(x_t, x_1)}dt =
\nonumber
\\
\frac{1}{\epsilon} \int_0^1 \int_{\mathbb{R}^D \times \mathbb{R}^D} \Big\{(\nabla_{\theta} g_{\theta}(x_t, t))^{\top}(g_{\theta}(x_t, t) - \frac{x_1 - x_t}{1 - t})  \Big\} d p_{T_{\pi}(x_t, x_1)} dt.
\end{eqnarray}

Now we recap the result of Theorem~\ref{th:hardsbm_drift}:
$
g_{\theta}(x, t) = \frac{1}{(1-t)} \big(\mathbb{E}_{x' \sim p^{\varphi_{\theta}}(x'|x_t)}[x'] - x_t \big).
$ Now we derive $\nabla_{\theta}g_{\theta}(x,t)$:
\begin{eqnarray}
\nabla_{\theta} g_\theta(x_t, t) =  \frac{1}{(1-t)} \big(\nabla_{\theta}\mathbb{E}_{x' \sim p^{\varphi_{\theta}}(x'|x_t)}[x'] \big) = \big[ Z(x_t, \varphi_{\theta}) \defeq \int_{\mathbb{R}^D} \exp\big( -\frac{\Vert x' - x_t \Vert^2}{2(1-t)\epsilon} \big)  \varphi_{\theta}(x') dx \big] =
    \nonumber
    \\
=  \frac{1}{(1-t)} \big( \nabla_\theta \int_{\mathbb{R}^D} x' \frac{\exp\big( -\frac{\Vert x' - x_t \Vert^2}{2(1-t)\epsilon } \big)  \varphi_{\theta}(x')}{Z(x_t, \varphi_\theta)} dx' \big) = 
    \nonumber
    \\
=  \frac{1}{(1-t)} \big(  \int_{\mathbb{R}^D} x' \nabla_\theta \big\{ \frac{\exp\big( -\frac{\Vert x' - x_t \Vert^2}{2(1-t)\epsilon } \big)  \varphi_{\theta}(x')}{Z(x_t, \varphi_\theta)}\big\}dx' \big)= \Big[ \nabla_\theta f_{\theta}(\cdot) = f_{\theta}(\cdot) \nabla_\theta \log f_{\theta}(\cdot) \Big] = 
   \label{eq:supp-log-der-trick-1}
    \\
=\frac{1}{(1-t)} \big( \nabla_\theta \mathbb{E}_{x' \sim p^{\varphi_{\theta}}(x'|x_t)} \big[ x' \nabla_\theta (-\frac{\Vert x' - x_t \Vert^2}{2(1-t)\epsilon } + \log \varphi_{\theta}(x') - \log Z(x_t, \varphi_\theta) )\big]\big) = 
    \nonumber
    \\
=\frac{1}{(1-t)} \big( \nabla_\theta \mathbb{E}_{x' \sim p^{\varphi_{\theta}}(x'|x_t)} \big[ x'( \nabla_\theta \big\{ \log \varphi_{\theta}(x') \big\}  - \nabla_\theta \log Z(x_t, \varphi_\theta) )\big]\big) = 
    \nonumber
    \\
=\frac{1}{(1-t)} \big( \nabla_\theta \mathbb{E}_{x' \sim p^{\varphi_{\theta}}(x'|x_t)} \big[ x' \nabla_\theta \big( \big\{ \log \varphi_{\theta}(x') \big\}  -  \frac{\int_{\mathbb{R}^D} \nabla_\theta \big( \exp \big\{ -\frac{\Vert x'' - x_t \Vert^2}{2\epsilon(1-t)} \big\} \varphi_{\theta}(x'') \big) dx'' }{Z(x_t, \varphi_\theta)} )\big]  \big) =  
    \nonumber
    \\
   =\Big[ \nabla_\theta f_{\theta}(\cdot) = f_{\theta}(\cdot) \nabla_\theta \log f_{\theta}(\cdot) \Big] =
    \label{eq:supp-log-der-trick-2}
    \\
=\frac{1}{(1-t)} \big( \nabla_\theta \mathbb{E}_{x' \sim p^{\varphi_{\theta}}(x'|x_t)} \big[ x'  \big( \nabla_\theta \big\{ \log \varphi_{\theta}(x') \big\}  -  \frac{\int_{\mathbb{R}^D} \exp \big\{ -\frac{\Vert x'' - x_t \Vert^2}{2\epsilon(1-t)} \big\} \varphi_{\theta}(x'') \nabla_\theta \log \varphi_{\theta}(x'')  dx'' }{Z(x_t, \varphi_\theta)}\big)\big] \big) =
    \nonumber
    \\
=\frac{1}{(1-t)} \big( \nabla_\theta \mathbb{E}_{x' \sim p^{\varphi_{\theta}}(x'|x_t)} \big[ x'  \big( \nabla_\theta \log \varphi_{\theta}(x')  -  \mathbb{E}_{x'' \sim p^{\varphi_{\theta}}(x''|x_t)} \big[ \nabla_\theta \log \varphi_{\theta}(x'') \big]\big)\big).
\end{eqnarray}
In lines \eqref{eq:supp-log-der-trick-1} and \eqref{eq:supp-log-der-trick-2}, we use the log-derivative trick.
\end{proof}
\clearpage

\begin{figure*}[!t]
% \vspace{-15mm}
\begin{subfigure}[b]{0.99 \linewidth}
\centering
\includegraphics[width=0.995\linewidth]{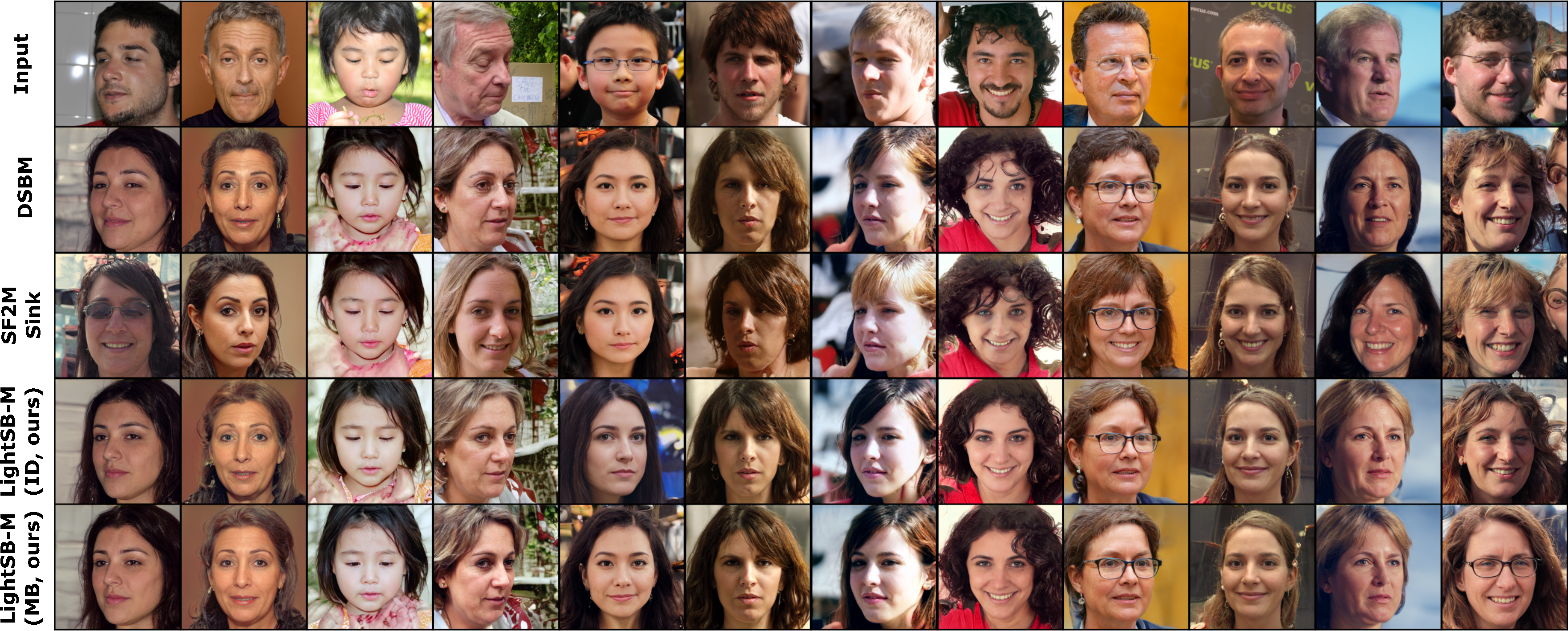}
\caption{\vspace{-1mm} \centering \textit{Man} $\rightarrow$ \textit{Woman}.}
\vspace{-1mm} % \label{fig:outdoor-church-comparison}
\end{subfigure}
\vskip\baselineskip
\begin{subfigure}[b]{0.99 \linewidth}
\centering
\includegraphics[width=0.995\linewidth]{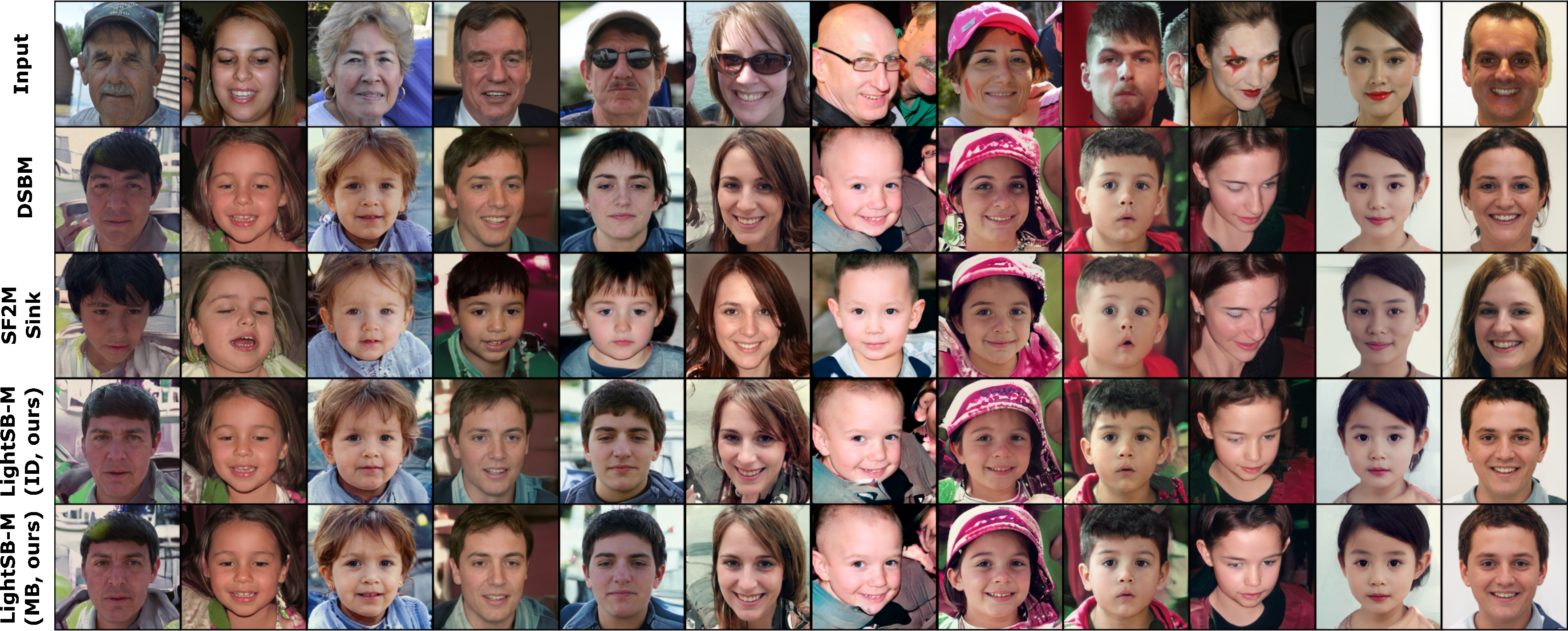}
\caption{\vspace{-1mm} \centering \textit{Adult} $\rightarrow$ \textit{Child}.}
\vspace{-1mm} % \label{fig:outdoor-church-comparison}
\end{subfigure}
\caption{\centering Additional examples of image-to-image translation.}\label{fig:additiona-exmp}
\end{figure*}

\begin{figure}
    \centering
    \includegraphics[width=0.90\linewidth]{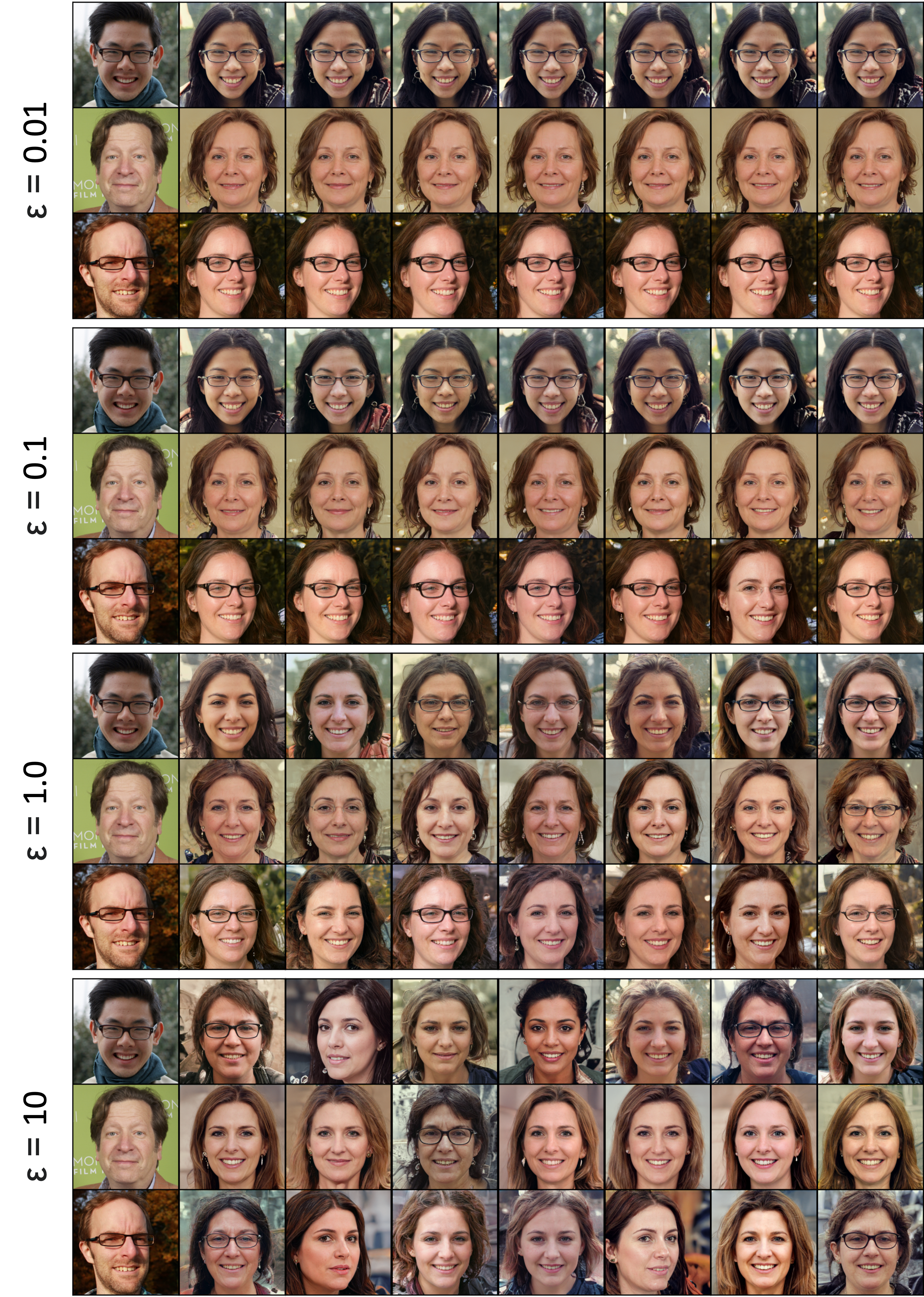}
    \caption{Image-to-image experiments with $\epsilon \in \{0.01, 0.1, 1, 10\}$.}
    \label{fig:appx_lightsbm_diversity}
\end{figure}

\clearpage

\end{document}